\documentclass[11pt]{article}

\usepackage[margin=1in]{geometry}
\usepackage{amsmath,amssymb,amsthm}
\usepackage{algorithm}
\usepackage{algorithmic}
\usepackage{booktabs}
\usepackage{graphicx}
\usepackage{array}
\usepackage{tikz}
\usetikzlibrary{arrows.meta,positioning}
\definecolor{splitgray}{RGB}{113,128,150}
\definecolor{acigold}{RGB}{214,158,46}
\definecolor{spectralgreen}{RGB}{56,161,105}
\definecolor{hybridorange}{RGB}{192,86,33}
\usepackage[authoryear,round]{natbib}
\usepackage{hyperref}
\newtheorem{assumption}{Assumption}
\newtheorem{theorem}{Theorem}

\newtheorem{corollary}{Corollary}
\newtheorem{proposition}{Proposition}
\theoremstyle{definition}
\newtheorem{remark}{Remark}

\begin{document}

\begin{center}
{\LARGE\bfseries Spectral Similarity Weights for Adaptive Conformal Prediction: A Bias-Variance Account and a Benchmark\par}
\vspace{1em}
{\large Jeffery Opoku\textsuperscript{a,*} and David Banahene\textsuperscript{b}\par}
\vspace{0.5em}
{\small
\textsuperscript{a}University of Texas Rio Grande Valley, School of Mathematical and Statistical Sciences, Edinburg, Texas, United States\\
\textsuperscript{b}Florida International University, Robert Stempel College of Public Health and Social Work, Department of Biostatistics, Miami, Florida, United States\\
\textsuperscript{*}Corresponding author. Email: \texttt{jeffery.opoku01@utrgv.edu}. ORCID: 0009-0002-4508-0027\\
David Banahene. Email: \texttt{dbana006@fiu.edu}
\par}
\end{center}
\vspace{1em}

\begin{abstract}
Conformal prediction intervals lose their guarantee when a series is not exchangeable. We study weighting calibration residuals by the similarity of local Fourier power features, combined with the adaptive level update of ACI. A two-term account of the coverage error, mismatch that falls and sampling error that rises as the bandwidth shrinks, holds with exogenous features and conditionally independent residuals, but not for features computed from the series itself, which is how the method is used. In four simulations, spectral weighting lowered the mean regime-level deviation of coverage by 17\% to 19\% when the features came from the response and by 75\% when they were exogenous, and not clearly under gradual drift. In each recurring-regime design a scaled-error baseline or Conformal PID control calibrated better. On nine public series with 36 forecast origins, spectral weighting with the plain weighted quantile did not differ from projected ACI in coverage error, and Conformal PID and scaled ACI were the best calibrated methods. The features are blind to seasonal phase, so the weights act like a recency kernel. A pure recency kernel did at least as well, especially with a finite-sample rank correction (mean absolute coverage error 0.039 against 0.090 for ACI; a post hoc check). Calendar phase weights narrowed German electricity intervals by 23\% at equal coverage, but were not better than ACI across the series. We suggest trying adaptive baselines first, and similarity weights only where regimes differ in spectral shape or season and the effective sample size stays large.
\medskip
\noindent\textbf{Keywords:} conformal prediction; adaptive conformal inference; time series; spectral features; effective sample size; benchmark
\end{abstract}

\section{Introduction}

Conformal prediction turns point forecasts into intervals with a coverage guarantee, whatever the forecaster \citep{lei2018distribution,angelopoulos2023gentle}. The guarantee needs exchangeable data, and a time series is not exchangeable. The next observation may come from a new season, a changed regime or a different frequency pattern, so an old residual may say little about the error now.

Two responses are common. One corrects the interval after errors occur, by changing the target level online as adaptive conformal inference (ACI) does \citep{gibbs2021adaptive}, or with stronger relatives such as AgACI and conformal PID control \citep{zaffran2022adaptive,angelopoulos2023pid}. These methods repair the long-run miss rate, but they do not say which past residuals are relevant. The other response selects residuals by recency or by similarity \citep{tibshirani2019conformal,barber2023conformal}. If a series moves between regimes with different oscillation patterns, residuals from an earlier time with the same pattern may help more than recent residuals from another one.

This paper asks whether the local frequency content of a series can pick out relevant past residuals, and whether that adds anything to feedback methods that already adapt. Each point gets a vector of normalized Fourier powers over a trailing window. Residuals are weighted by a Gaussian kernel in the distance between these vectors, the weighted quantile sets the interval, and an ACI-type update moves the level. The modulus of a Fourier transform does not change when the window is shifted, so the features describe which frequencies dominate and not the phase within a season. We therefore speak of similarity of spectral shape, not of season matching (Section~\ref{sec:features}).

The paper makes four contributions, each with its limits stated.

\begin{enumerate}
\item \emph{A two-term account of coverage error} (Section~\ref{sec:theory}). Under a working model with exogenous features, the conditional coverage deficit is bounded by a mismatch term and a sampling term set by the effective sample size. A cluster form of the mismatch term, and an application of the bound of \citet{barber2023conformal}, cover recurring regimes. The account does not apply to features computed from past responses, and the long-run identity of ACI does not depend on the weights.
\item \emph{Simulations that separate the mechanism from the features} (Section~\ref{sec:sims}). Four designs, an oracle that knows the regime, a scaled-error baseline and Conformal PID, and a check of how loose the bound is when its assumptions hold.
\item \emph{A benchmark with intervals} (Section~\ref{sec:real}). Four chronological splits and a rolling-origin study of nine public series at 36 origins, with paired differences, two bootstrap intervals and an exact sign-flip test. The short test blocks of three series are reported openly, and we check the quantile convention, phase-aware weights and a pure recency kernel.
\item \emph{A correction of the effective sample size safeguard} (Section~\ref{sec:safeguard}). The rule as coded leaves the bandwidth unchanged on the series where an earlier version reported that it helped. We explain why, show that on the real series the weights behave like a recency kernel, and test an online variant that was added after the failure was seen.
\end{enumerate}

The main findings follow. In simulations with recurring regimes, spectral weighting reduces regime-level miscalibration relative to equal weights, by about a fifth when features come from the response and by three quarters when they are exogenous. It does not reach the oracle, and in each recurring-regime design some baseline without spectral information does better, at the cost of wider intervals. On nine public series, spectral weighting is not distinguishable from projected ACI in coverage or conditional calibration, and a recency kernel does as well or better. Weights on calendar phase help on one series, German electricity, where intervals are 23\% narrower at equal coverage in the single split, and not on average. Conformal PID and scaled ACI come closest to the target coverage. Spectral weighting is therefore a reasonable idea with a clean account of its error, and we found no setting with real data where it is the better choice.

This version replaces an earlier arXiv version (arXiv:2606.15950). It withdraws several statements of that version, which we point out where they arise, and it adds the baselines, simulations and intervals described above. The method shares its core with a companion manuscript on streaming data (Section~\ref{sec:companion}). The supplement (Online Resource 1) holds the proofs, the simulation details and the complete result tables. Its sections, tables and figures carry an S prefix.

\section{Related Work}

\subsection{Conformal prediction and its limits for ordered data}

Split conformal prediction gives finite-sample marginal coverage under exchangeability and can sit on top of almost any prediction model \citep{lei2018distribution,angelopoulos2023gentle}. Jackknife+ and conformalized quantile regression refine the idea \citep{barber2021predictive,romano2019conformalized}. For ordered data, exchangeability is the weak point. A residual from last summer may say little about next winter.

\citet{tibshirani2019conformal} reweight calibration scores under covariate shift. \citet{barber2023conformal} study conformal prediction with fixed weights and without exchangeability. Their coverage bound loses the weighted total variation distance between the test residual and each calibration residual. Localized conformal methods use nearby calibration points to improve conditional behaviour \citep{guan2023localized,hore2023conformal}. This paper also localizes, but the neighbourhood is defined by the local frequency content of the series. It is closest to the weighting in \citet{barber2023conformal}, and Corollary~\ref{cor:barber} applies their bound to a model with recurring regimes.

\subsection{Conformal methods for time series}

EnbPI and sequential predictive conformal inference (SPCI) use the dependence of residuals in time \citep{xu2021enbpi,xu2022sequential}. SPCI fits a model to past residuals to predict the quantile of the next one, so it is the closest published competitor to any method that tries to learn structure in the residuals. We did not run EnbPI or SPCI. Both need a bootstrap ensemble or a quantile forest around the point forecaster, which would change the forecasting model that all other methods share. The conformal PID score forecaster, described next, plays a similar role and was run.

Adaptive conformal inference (ACI) updates the miscoverage level after each observation \citep{gibbs2021adaptive}. AgACI combines ACI experts with different learning rates \citep{zaffran2022adaptive}, and strongly adaptive online conformal prediction pursues the same aim with online learning tools \citep{bhatnagar2023improved}. Conformal PID control adds a tangent integrator and a score forecaster to a quantile tracker, and was designed for seasonal and trending errors \citep{angelopoulos2023pid}. Other work uses dynamic programming or related ideas for the interval width \citep{yang2024bellman,stocker2025gentle}, and copula or dual-splitting methods for multi-step forecasts \citep{sun2022copula,yu2025dual}. These last methods answer different questions and are not compared here.

Conformal PID control and AgACI are the strongest baselines for our question. They adapt to a changing error level through feedback, without selecting which past residuals are relevant. Spectral weighting selects residuals and then adds the same kind of feedback. Whether the selection adds anything beyond the feedback is the empirical question of this paper.

\subsection{Local spectra}

The features are local summaries of oscillatory behaviour, in the spirit of locally stationary time series \citep{dahlhaus1997fitting,dahlhaus2009empirical}. We do not assume local stationarity or any model for the spectrum. The features are coordinates for similarity, and nothing more.

\subsection{Relation to the companion paper}
\label{sec:companion}

The same authors have a companion manuscript on streaming data under drift \citep{opoku2026dasc}. The two share the core of the method: causal local Fourier features, Gaussian kernel weights on calibration residuals, an effective sample size through Hoeffding's inequality, a Lipschitz condition on conditional residual laws, and the adaptive level update. Table~\ref{tab:companion} lists what differs. The difference is real but narrow, and a reader could reasonably see the two as one line of work. The companion paper also states its Lipschitz condition for features computed from past responses. Its theory has not been rechecked here against Remark~\ref{rem:excluded}.

\begin{table}[h]
\centering
\caption{Differences between this paper and the companion manuscript \citep{opoku2026dasc}. Entries for the companion are taken from its text.}
\label{tab:companion}
\scriptsize
\setlength{\tabcolsep}{4pt}
\begin{tabular}{@{}p{0.17\linewidth}p{0.38\linewidth}p{0.38\linewidth}@{}}
\toprule
 & This paper & Companion (drift-aware) \\
\midrule
Setting & Ordered series with recurring seasons or regimes; one-step forecasts; rolling-origin and single chronological splits & Streaming data with drift \\
Features & Normalized powers at the first six Fourier frequencies of a causal window (window length set per series) & Six-band periodogram, window length 64, $\ell_2$ normalized \\
Weights & $\exp\{-d^2/(2b^2)\}$, bandwidth chosen on a validation part of the calibration block, plus an effective sample size rule & $\exp\{-d^2/h^2\}$, bandwidth fixed at 0.55 \\
Drift handling & None in the method; effective sample size used as a diagnostic and a bandwidth rule & Drift score and drift-gated calibration window \\
Level update & Projected ACI, $\alpha_t\in[0.005,0.50]$ & Projected ACI, $\alpha_t\in[0.01,0.35]$ \\
Theory & Working model with exogenous features; fixed-time bound; effective sample size concentration; cluster mismatch bound; bandwidth rate; pathwise identity & Lipschitz-based instantaneous bound; data-dependent bound; mixing concentration; gate trade-off; long-run result \\
Evidence & Four simulations; nine public series with four origins each; four single splits; Conformal PID and AgACI compared & Synthetic streams and stress tests; household electricity, hourly weather, financial volatility \\
\bottomrule
\end{tabular}
\end{table}

\section{Method}
\label{sec:method}

This section states the method as it is implemented. The pieces are separate: a weighted conformal quantile, spectral features that define the weights, a bandwidth rule, and an adaptive level update. A user can change the point forecaster, the score or the features without changing the rest.

\subsection{Weighted conformal quantile}

Let \(\widehat f\) be a point forecaster fitted on training data. On a calibration block, define the scores \(R_i=|Y_i-\widehat f(X_i)|\). Split conformal prediction returns \(\widehat f(X_t)\pm q_\alpha\), where \(q_\alpha\) is an empirical quantile of the scores, and treats all calibration scores alike. That is reasonable when scores are exchangeable and crude when the calibration block holds several regimes.

For a test index \(t\), let \(w_{t,i}\ge0\) be weights over the pool \(\mathcal I\) of available scores, with \(\sum_iw_{t,i}=1\). The weighted quantile is
\[
 q_t(a)=\inf\Bigl\{q:\sum_{i\in\mathcal I}w_{t,i}\mathbf 1\{R_i\le q\}\ge1-a\Bigr\},
\]
and the interval is \([\widehat f(X_t)-q_t(a),\,\widehat f(X_t)+q_t(a)]\). The implementation uses this quantile without the extra point mass at infinity that finite-sample conformal guarantees need (see Corollary~\ref{cor:barber}). Section~\ref{sec:theory} says which statements hold for which version.

\subsection{Local spectral features and weights}
\label{sec:features}

For each index \(i\), let \(Z_i\in\mathbb R^6\) be computed from the \(m\) observations before time \(i\) (zero padded at the start of the series). We remove the window mean, take the squared modulus of the discrete Fourier transform at the first six positive frequencies, and divide by their sum, so that the features describe relative frequency content and not scale. The window length \(m\) is set per series (96, 208, 730, 180, 156, 55, 72, 40 and 40 for the nine series of Section~\ref{sec:rolling}; 96 in Simulations 1, 1b and 2; 48 in Simulation 3, where the features come from an exogenous covariate). The feature at time \(i\) uses no observation from time \(i\) or later.

The squared modulus of the discrete Fourier transform is unchanged by a circular shift of the window. For a window that holds whole seasonal cycles, the features at a given season and at the same season one year later therefore differ only through trend, noise and edge effects, and they do not record the phase within the season. They record which frequencies dominate and how strongly, so they can tell apart regimes that differ in oscillation period or in the balance of frequencies, as in our simulations. They cannot match a season to the same season of an earlier year. Section~\ref{sec:lag} measures how this plays out on the real series. The features also keep only ordinates one to six of the window. When the window is more than six seasonal periods long, as for the utilities, German electricity and GDP series, the seasonal frequency itself is not among them.

With bandwidth \(b>0\) the weights are Gaussian,
\[
 w_{t,i}=\frac{\exp\{-\|Z_t-Z_i\|_2^2/(2b^2)\}}{\sum_{j\in\mathcal I}\exp\{-\|Z_t-Z_j\|_2^2/(2b^2)\}} .
\]
The effective sample size is \(n_{\mathrm{eff}}(t)=1/\sum_iw_{t,i}^2\). It equals the pool size under equal weights and is close to one when a single score dominates.

\subsection{Bandwidth rules}
\label{sec:bandwidth-rules}

\emph{Validation rule (used in all spectral methods).} The first 70\% of the calibration block is the reference set and the last 30\% the validation set. For each bandwidth in a grid we compute coverage and mean width on the validation set. Among bandwidths with coverage at least \(1-\alpha\) we take the one with the smallest width. If none qualifies we take the one with coverage closest to \(1-\alpha\). This rule uses calibration residuals, so the weights are not fixed in advance in the sense of Assumption~\ref{ass:working}(b). The grids are \(\{0.04,0.06,0.08,0.12,0.18,0.25,0.35\}\) in the single splits (with 0.50 added in the rolling-origin study), \(\{0.12,0.18,0.25,0.35\}\) in Simulations 1, 1b and 3, and \(\{0.04,0.06,0.08,0.12,0.18,0.25\}\) in Simulation 2.

\emph{Calibration safeguard (the rule as coded).} Starting from the validated bandwidth, we compute for every calibration point its effective sample size against the other calibration points (leave one out). If the median is below a floor \(n_{\min}=20\), we move to the next wider bandwidth in the grid and repeat. If no bandwidth in the grid meets the floor, we use equal weights. We cannot tell from the project record whether the floor was fixed before or after the utilities result in Section~\ref{sec:real} was seen, so we treat the rule as post hoc.

\emph{Online gate (exploratory).} At each test time we compute the effective sample size of that test point against the current pool, for each grid bandwidth not smaller than the validated one, and use the first for which it reaches 20. If none does we use equal weights. It uses features only, so it is available at prediction time. We added it after the safeguard above turned out not to change the utilities result (Section~\ref{sec:safeguard}), so it is a post hoc repair, and results for it are labelled exploratory.

\subsection{Adaptive level update}

The adaptive version replaces \(\alpha\) by a level \(\alpha_t\) with \(\alpha_1=\alpha\). After observing \(Y_t\), with \(M_t=\mathbf 1\{Y_t\notin\widehat C_t\}\), the level is updated by \(\alpha_{t+1}=\alpha_t+\gamma(\alpha-M_t)\), and the new pair \((R_t,Z_t)\) joins the pool (Algorithm~\ref{alg:sa}). The experiments project \(\alpha_t\) to \([0.005,0.50]\), so every interval is finite. The pathwise identity of Proposition~\ref{prop:adaptive-long-run} is for the unprojected update.

\begin{algorithm}[h]
\caption{Spectral adaptive conformal prediction}
\label{alg:sa}
\begin{algorithmic}[1]
\STATE Fit \(\widehat f\) on training data; compute calibration scores \(R_i\) and features \(Z_i\).
\STATE Choose the bandwidth (validation rule, and safeguard if used); set \(\alpha_1=\alpha\).
\FOR{each test time \(t\)}
    \STATE Compute weights \(w_{t,i}\) from \(Z_t\) and the pool features; compute \(q_t(\alpha_t)\).
    \STATE Return \([\widehat f(X_t)-q_t(\alpha_t),\,\widehat f(X_t)+q_t(\alpha_t)]\).
    \STATE Observe \(Y_t\); set \(M_t\); update \(\alpha_{t+1}=\Pi_{[0.005,0.50]}\{\alpha_t+\gamma(\alpha-M_t)\}\).
    \STATE Add \((R_t,Z_t)\) to the pool.
\ENDFOR
\end{algorithmic}
\end{algorithm}

\subsection{Comparison methods}
\label{sec:baselines}

All methods share the point forecaster, the scores and the calibration block, and all adaptive methods start from \(\alpha_1=\alpha\) with the same projection.

\begin{itemize}
\item \emph{Split}: equal weights, fixed level.
\item \emph{Projected ACI}: equal weights, adaptive level.
\item \emph{Scaled ACI}: ACI on scores divided by a causal scale, the mean of the previous 20 observed scores in the calibration and test block. This follows the idea of locally weighted conformal prediction \citep{lei2018distribution} in its simplest form. It tracks the error scale and uses no spectral information.
\item \emph{AgACI}: ACI experts with \(\gamma\in\{0.001,0.005,0.01,0.02,0.05\}\) whose lower and upper endpoints are combined by exponentially weighted aggregation with pinball losses (learning rate 0.08), as in \citet{zaffran2022adaptive}. Our code is a reimplementation for symmetric intervals.
\item \emph{Conformal PID}: a quantile tracker plus a tangent integrator \citep{angelopoulos2023pid}, written from the published algorithm. The tracker starts at the split radius and its step is 0.1 times the range of the last 100 scores. The integrator constant \(K_I\) is the largest calibration score and the saturation constant \(C_{\mathrm{sat}}\) follows the paper's heuristic with the test length as horizon and \(\delta=0.05\). The radius is kept in \([0,10K_I]\); this bound is active on 0.006\% of test points on average for plain PID (at most 0.23\% at one origin) and on 0.03\% with the scorecaster (at most 0.54\%). With the scorecaster, a Theta model forecasts the next score from all scores seen so far and the radius moves by the change in that forecast from its first value. We checked the tracker and integrator against the authors' public code on three synthetic sequences when started at zero, with their one-step integrator lag and no clipping; the radii agreed exactly (supplement). Our scorecaster alignment and warm start differ from the public code, and the PID results below come from our implementation, not theirs.
\item \emph{Spectral ACI} (validation bandwidth), \emph{Safe spectral ACI} (with the calibration safeguard), \emph{Gated spectral ACI} (online gate, exploratory) and \emph{Spectral AgACI} (the AgACI experts with spectral weights). 
\item \emph{Oracle regime ACI} (simulations only): equal weights over pool scores from the same true regime as the test point. It is an upper reference for any similarity weighting, not a usable method.
\end{itemize}

The weighted quantile of the spectral methods is the plain one defined above. For equal weights it is the order statistic of rank \(\lceil N(1-a)\rceil\), while projected ACI uses rank \(\lceil(N+1)(1-a)\rceil\). The two differ by one order statistic, which matters when the pool has 20 to 30 points. After a review pointed this out, we added a post hoc check, labelled \emph{rank corrected}, in which the spectral quantile has the finite-sample form of Corollary~\ref{cor:barber}. Its target level is \((1-a)(1+1/S)\), with \(S\) the sum of the unnormalised kernel values. With equal weights it reproduces projected ACI exactly, and a unit test confirms this. With concentrated weights it is more than a one-rank shift, because the target level exceeds one when \(S<(1-a)/a\), which is below 9 for \(a=0.1\). The exact rule then gives an infinite interval, and the code returns the largest pool residual.

Phase-aware and recency variants (Phase ACI, Spectral + phase ACI, Recency ACI) were added late and are defined in Section~\ref{sec:phase}. None of the methods was tuned per series. The learning rate is \(\gamma=0.005\) in the rolling-origin study and \(0.02\) in the simulations and single splits, as in the original experiments, and Section~\ref{sec:rolling} reports a check at \(\gamma=0.02\).

\section{Theory}
\label{sec:theory}

The results fall into two groups. The first group is a probability bound at one prediction time. It needs a working model, stated below, and that model does not hold for the features used in our experiments. The second group is a deterministic identity for the adaptive update. It holds for every data path, and it says nothing about the weights. We keep the groups apart, because the second could otherwise be read as support for the first.

\subsection{A working model with exogenous features}

Fix a prediction time \(t\). Let \(\mathcal I_{t-1}\) be the pool of residuals available before time \(t\). In the fixed-time results this is the calibration block. In the adaptive method it also contains the earlier test residuals, because Algorithm~\ref{alg:sa} adds each new pair to the pool. Let \(\mathcal H_t\) be the \(\sigma\)-field generated by the training sample (so the fitted forecaster is fixed) and by the feature vectors \(Z_i\) for \(i\le t\). It contains no response from the calibration or test period.

\begin{assumption}[Working model]
\label{ass:working}
(a) Each \(Z_i\), \(i\le t\), is \(\mathcal H_t\)-measurable and is not a function of calibration or test responses. (b) The bandwidth \(b\) is fixed, or is chosen on data outside the pool, so the weights \(w_{t,i}\) are \(\mathcal H_t\)-measurable. (c) Conditional on \(\mathcal H_t\), the residuals \(R_i\), \(i\in\mathcal I_{t-1}\), and \(R_t\) are independent. Write \(F_i\) for the conditional law of \(R_i\) and \(F_t\) for that of \(R_t\).
\end{assumption}

For \(i\in\mathcal I_{t-1}\), write
\[
 \overline F_t(r)=\sum_{i\in\mathcal I_{t-1}}w_{t,i}F_i(r),
 \qquad
 \widehat F_t(r)=\sum_{i\in\mathcal I_{t-1}}w_{t,i}\mathbf 1\{R_i\le r\}.
\]

\begin{remark}[What the working model excludes]
\label{rem:excluded}
The spectral features in the experiments are computed from past responses, so part (a) fails for them. Conditioning on those responses does not repair this. If past responses were placed in \(\mathcal H_t\), each \(R_i\) would be known, \(F_i(r)=\mathbf 1\{R_i\le r\}\), and \(\sup_r|F_i(r)-F_t(r)|\ge 1/2\) whenever \(F_t\) is continuous. The smoothness condition below would then force \(L\|Z_i-Z_t\|\ge 1/2\) for every \(i\), and the bound would say nothing. The same holds for a bandwidth chosen on the pool residuals, which breaks part (b); the experiments choose it that way. For these cases the results below describe the two sources of error, a mismatch term and a sampling term. They are not coverage guarantees for the method as run.
\end{remark}

\begin{assumption}[Spectral smoothness]
\label{ass:smooth}
There is a constant \(L>0\) such that
\(\sup_r|F_i(r)-F_t(r)|\le L\|Z_i-Z_t\|_2\)
for all \(i\in\mathcal I_{t-1}\).
\end{assumption}

\begin{assumption}[Weighted empirical concentration]
\label{ass:conc}
With probability at least \(1-\delta\), the event
\(E_t=\{\sup_r|\widehat F_t(r)-\overline F_t(r)|\le\varepsilon_n\}\) holds.
\end{assumption}

The next proposition gives the scale of \(\varepsilon_n\) on a finite grid. It does not prove uniform concentration over every real threshold.

\begin{proposition}[Effective sample size and concentration]
\label{prop:neff-concentration}
Under Assumption~\ref{ass:working}, let
\(n_{\mathrm{eff}}(t)=1/\sum_{i\in\mathcal I_{t-1}}w_{t,i}^2\).
For any fixed threshold \(r\) and any \(u>0\),
\[
 \mathbb P\bigl\{|\widehat F_t(r)-\overline F_t(r)|>u\mid\mathcal H_t\bigr\}
 \le 2\exp\{-2u^2 n_{\mathrm{eff}}(t)\}.
\]
Hence, for a finite grid \(\mathcal G\) of thresholds, with conditional probability at least \(1-\delta\),
\[
 \max_{r\in\mathcal G}|\widehat F_t(r)-\overline F_t(r)|
 \le\sqrt{\frac{\log(2|\mathcal G|/\delta)}{2n_{\mathrm{eff}}(t)}}.
\]
\end{proposition}

Define the weighted mean spectral distance
\[
 B_t=\sum_{i\in\mathcal I_{t-1}}w_{t,i}\|Z_i-Z_t\|_2 .
\]
For a level \(a\in(0,1)\), let
\(q_t(a)=\inf\{q:\widehat F_t(q)\ge 1-a\}\).

\begin{theorem}[Fixed-time coverage under the working model]
\label{thm:fixed}
Assume Assumptions~\ref{ass:working} to~\ref{ass:conc}. Let \(a\) be any level that is measurable with respect to \(\mathcal H_t\) and the pool residuals, for example the level \(\alpha_t\) produced by the adaptive update. On the event \(E_t\),
\[
 \mathbb P\bigl\{R_t\le q_t(a)\mid\mathcal H_t,(R_i)_{i\in\mathcal I_{t-1}}\bigr\}
 \ge 1-a-LB_t-\varepsilon_n .
\]
\end{theorem}

The bound is stated at the level \(a=\alpha_t\) that the method actually uses. It does not say that the long-run average of \(\alpha_t\) equals \(\alpha\); that is the job of Proposition~\ref{prop:adaptive-long-run} below.

\begin{remark}[Reading the bound]
The term \(LB_t\) is the mismatch between the weighted pool and the test point. The term \(\varepsilon_n\) is sampling error. A smaller bandwidth lowers \(B_t\) and raises \(\varepsilon_n\), because it lowers \(n_{\mathrm{eff}}(t)\). Neither term is observable in practice, since \(L\) is not known, and the bound can be far from the realized coverage. Section~\ref{sec:sim3} reports how loose it is in a simulation where Assumption~\ref{ass:working} holds.
\end{remark}

\begin{corollary}[Bandwidth rate]
\label{cor:bandwidth-rate}
Suppose that, over a bandwidth range \([b_{\min},b_{\max}]\), there are constants \(c_1,c_2>0\) with \(B_t\le c_1 b\) and \(n_{\mathrm{eff}}(t)\ge c_2Nb^{d}\), where \(N=|\mathcal I_{t-1}|\) and \(d\) is the dimension of the feature space (five for our features, because the six powers sum to one). For the grid version of the quantile, the smallest point \(q\) of a grid \(\mathcal G\) fixed in advance with \(\widehat F_t(q)\ge1-a\),
\[
 1-a-\mathbb P\{R_t\le q_t(a)\mid\mathcal H_t,(R_i)\}
 \le Lc_1b+\sqrt{\frac{\log(2|\mathcal G|/\delta)}{2c_2Nb^{d}}}.
\]
The choice \(b\asymp N^{-1/(d+2)}\) makes both terms of order \(N^{-1/(d+2)}\), up to the logarithmic factor.
\end{corollary}

\subsection{Mismatch in a cluster model}

Smoothness is the usual way to bound the mismatch term. A simpler case is a recurring regime model, in which pool residuals either have the same law as the test residual or do not.

\begin{proposition}[Cluster mismatch]
\label{prop:cluster}
Let \(M\subset\mathcal I_{t-1}\) be the set of indices with \(F_i\ne F_t\), and let \(W_M=\sum_{i\in M}w_{t,i}\). Then
\(\sup_r|\overline F_t(r)-F_t(r)|\le W_M\).
Suppose the weights are Gaussian with bandwidth \(b\), and that there are sets \(S_0\) (nonempty, with \(F_i=F_t\)) and \(M_0\subset M\) such that \(\|Z_i-Z_t\|\le\rho\) for \(i\in S_0\) and \(\|Z_i-Z_t\|\ge\Delta>\rho\) for \(i\in M_0\). Let \(B=M\setminus M_0\) be the remaining mismatched indices, for which no separation is assumed. Then
\[
 W_M\le W_B+\frac{|M_0|}{|S_0|}\exp\Bigl\{-\frac{\Delta^2-\rho^2}{2b^2}\Bigr\},
 \qquad W_B=\sum_{i\in B}w_{t,i}.
\]
\end{proposition}

Combining Proposition~\ref{prop:cluster} with the proof of Theorem~\ref{thm:fixed} gives the same coverage bound with \(W_M\) in place of \(LB_t\). On the event \(E_t\), \(\mathbb P\{R_t\le q_t(a)\mid\mathcal H_t,(R_i)\}\ge1-a-W_M-\varepsilon_n\), and no Lipschitz constant is needed.

The set \(B\) is not a technicality. In Simulation~1 the feature window is 96 points and the regimes last 120 points, so most windows cross a regime boundary and belong, for this purpose, to \(B\). The proposition says that the weight on well separated wrong regimes falls exponentially in \((\Delta^2-\rho^2)/b^2\), and that nothing is gained against the boundary mass \(W_B\) by shrinking the bandwidth alone.

\begin{corollary}[Coverage without conditional independence]
\label{cor:barber}
Let the features and weights be fixed in advance, as in parts (a) and (b) of Assumption~\ref{ass:working}. Let \(k_i=\exp\{-\|Z_i-Z_t\|^2/(2b^2)\}\), \(\tilde w_i=k_i/(1+\sum_jk_j)\), and let \(\tilde q_t\) be the \((1-\alpha)\) quantile of \(\sum_{i\in\mathcal I_{t-1}}\tilde w_i\delta_{R_i}+\tilde w_t\delta_{+\infty}\), with \(\tilde w_t=1/(1+\sum_jk_j)\). Let \(M'\) be the set of indices \(i\) for which swapping \(R_i\) and \(R_t\) changes the joint law of the residuals, and suppose \(M'\subset M\) (in a regime model, \(M'\) is the set of other-regime indices). Write \(W_{M'}=\sum_{i\in M'}w_{t,i}\). Then
\[
 \mathbb P\{R_t\le\tilde q_t\}\ge 1-\alpha-\sum_{i\in M'}\tilde w_i\ge 1-\alpha-W_{M'} .
\]
\end{corollary}

Corollary~\ref{cor:barber} applies the coverage bound for nonexchangeable split conformal prediction in \citet{barber2023conformal}, with total variation distance at most one for the indices in \(M'\) and zero for the others. We use their bound and do not prove it again. It needs no independence and no smoothness, but it needs fixed weights and a fixed level, and it uses the quantile with a point mass at \(+\infty\). The method in the experiments omits that point mass. With uniform weights and \(M'=\emptyset\) the corollary gives the usual split conformal guarantee \(1-\alpha\) for exchangeable scores, using the order statistic of rank \(\lceil(N+1)(1-\alpha)\rceil\).

\subsection{The adaptive update}

\begin{proposition}[Long-run calibration of the adaptive update]
\label{prop:adaptive-long-run}
Run \(\alpha_{t+1}=\alpha_t+\gamma(\alpha-M_t)\) for \(t=1,\dots,T\) with \(\alpha_1\in[0,1]\) and \(0<\gamma\le1\), without truncation. Let the interval be the whole real line when \(\alpha_t\le0\) (so \(M_t=0\)) and empty when \(\alpha_t\ge1\) (so \(M_t=1\)). Then
\[
 \frac1T\sum_{t=1}^TM_t=\alpha+\frac{\alpha_1-\alpha_{T+1}}{\gamma T},
\]
the iterates stay in \([-\gamma,1+\gamma]\), and
\(|T^{-1}\sum_tM_t-\alpha|\le(1+\gamma)/(\gamma T)\)
for every data sequence, for every choice of weights, and for every forecaster.
\end{proposition}

\begin{remark}[What the identity does and does not say]
The argument is that of \citet{gibbs2021adaptive}. It does not use the weights at all. Equal weights, spectral weights and any other weights give the same bound, so the identity gives no support for spectral weighting. It is also a long-run statement. It says nothing about coverage inside a season or a regime, which is why we report regime-level coverage. With \(T\) test points and \(\gamma=0.005\), the bound equals \((1+\gamma)/(\gamma T)\), which exceeds one for \(T<201\). For the short test blocks of the rolling-origin study (10 to 31 points for three of the nine series) it is vacuous.
\end{remark}

\begin{remark}[Projected update]
The experiments keep \(\alpha_t\) in \([0.005,0.50]\) so that every interval is finite. Projection adds a boundary correction to the telescoping sum, so Proposition~\ref{prop:adaptive-long-run} does not apply to the rule as run. We report how often the projection is active.
\end{remark}

\begin{remark}[Distance lemma]
\label{rem:distance}
For any non-increasing kernel, the weighted mean distance \(B_t\) is no larger than the unweighted mean distance \(N^{-1}\sum_id_i\), where \(d_i=\|Z_i-Z_t\|\) (Chebyshev's sum inequality; proof in the supplement). This compares two upper bounds. It says nothing about which interval covers better, because concentrated weights lower \(n_{\mathrm{eff}}(t)\).
\end{remark}

\section{Experimental design and reporting}
\label{sec:design}

\subsection{What was planned and what was added}

The two original simulations, the four single chronological splits and the rolling-origin study of nine series were designed in the first version of this paper. After that version was reviewed, and knowing its results, we added Simulations~1b and~3; scaled ACI, Conformal PID and spectral AgACI as comparison methods; the online effective sample size gate; the regime-level and phase-bin measures of conditional calibration; and the interval estimates for all comparisons. After a second review we also added the finite-sample rank check of the weighted quantile, the feature lag diagnostic (Section~\ref{sec:lag}) and the phase-aware and recency weights (Section~\ref{sec:phase}), which we ran on every experiment. Each is labelled as added. The conditional-calibration measures were chosen after we saw that overall coverage hid regime differences, so they are post hoc. Several comparisons are reported without correction for multiple testing. The intervals describe the data and do not form a family of confirmed findings.

\subsection{Measures}

For a test block we report coverage and the mean interval width. The rolling-origin study also reports the absolute coverage error \(|\text{coverage}-0.90|\) per origin and an interval score
\[
 \mathrm{IS}_t=(u_t-l_t)+\frac{2}{\alpha}\{(l_t-e_t)_++(e_t-u_t)_+\},
\]
where \(e_t\) is the forecast error and \([l_t,u_t]\) is the interval. We divide the score and the width by the median calibration score at that origin, so that series in different units can be averaged, with equal weight for each origin whatever the length of its test block.

Overall coverage can hide where a method fails, so we also measure conditional calibration. In Simulations~1, 1b, 2 and~3 the regime of every test point is known. For each repetition we compute coverage inside each regime and report the mean over regimes of the absolute deviation from 0.90 (the regime deviation) and the largest shortfall \(0.90-\min_{\text{regime}}\text{coverage}\). The real series have no regimes, so we use seasonal phase bins. The phase of test point \(t\) is \(t \bmod p\) for the nominal period \(p\) of the series, cut into \(\min(p,4)\) bins of nearly equal length. We average the absolute deviation over bins that hold at least five test points, and keep an origin only if two or more bins qualify. This removes the four Nile flow origins (10 test points each), so the phase-bin comparison uses 32 origins. The bins were fixed before the phase-bin results were computed.

\subsection{Intervals}

\emph{Simulations.} Repetitions are independent, so we report Monte Carlo standard errors and, for paired differences from projected ACI, normal 95\% intervals on the mean over 250 repetitions.

\emph{Single splits.} Test errors are dependent in time. We use a moving block bootstrap with blocks of 20 test points and 2000 resamples, and the same resampled blocks for both methods in a paired difference. The intervals are approximate, and too narrow if dependence extends beyond 20 points.

\emph{Rolling origins.} The four origins of a series overlap in time and share the series, so they are not independent. For each mean paired difference we give two intervals. The origin bootstrap resamples the 36 origins as if independent, and is too narrow. The series bootstrap resamples the nine series and then the origins within each chosen series. It respects the dependence but rests on nine clusters, so it is rough. We also report an exact sign-flip test on the nine series means, which assumes no distribution. With nine series the smallest attainable two-sided \(p\)-value is \(2/512\approx0.004\). The sign-flip test is our primary summary.

\subsection{Short test blocks}

The test block is 10\% of a series. For three series it is short: 10 points for Nile flow, 20 for U.S. real GDP and 31 for sunspots. The other six series have 73 (El Ni\~no), 105 (utilities), 146 (Seattle), 187 (gasoline), 228 (Mauna Loa) and 438 (German electricity) points. With \(\gamma=0.005\) the level \(\alpha_t\) can move by at most \(0.05\) in 10 steps, so every ACI-type method on the short series is close to split conformal, and the bound of Proposition~\ref{prop:adaptive-long-run} is vacuous. Averaging over origins also gives the short blocks the same weight as the 438-point blocks. We therefore report the main comparisons for all nine series and again for the six series with at least 40 test points, and we repeat the study at \(\gamma=0.02\).

\section{Simulations}
\label{sec:sims}

\subsection{Designs}

We use four designs, each with 250 repetitions, 480 training points, 960 calibration points and 960 test points, and a target coverage of 0.90. Details are in the supplement.

\begin{itemize}
\item \emph{Simulation 1 (original): recurring regimes.} Four regimes in blocks of 120 points repeat every 480 points. They differ in oscillation period and noise scale (0.28 to 1.25). The forecaster is a harmonic regression with a fixed period, which cannot follow all regimes. The features come from the response.
\item \emph{Simulation 2 (original): gradual change.} The local frequency drifts smoothly and the noise scale grows with it. There are no blocks. The regime variable is a quartile bin of the local frequency.
\item \emph{Simulation 1b (added): equal scale, different shape.} The regimes and blocks are those of Simulation~1, but the mean is known and the noise has unit variance in every regime with a different shape: Gaussian, heavy tailed, uniform, bimodal. A method that tracks only the residual scale cannot separate the regimes. We built this design to test the mechanism, and it favours spectral weighting by construction.
\item \emph{Simulation 3 (added): exogenous features.} The regime is carried by the spectrum of an exogenous covariate that the response does not feed back into, and the response is independent noise whose scale depends on the regime. The working model of Assumption~\ref{ass:working} holds here.
\end{itemize}

Besides the methods of Section~\ref{sec:baselines}, we include an oracle that weights only pool residuals from the same true regime as the test point. It is not a usable method, and it shows how much any similarity weighting could gain if the regime were known.

\subsection{Results}

Table~\ref{tab:sim} reports overall coverage and width, and the conditional measures of Section~\ref{sec:design}: the mean over regimes of the absolute deviation of coverage from 0.90, and the largest regime shortfall. The difference from projected ACI is paired by repetition. Overall coverage is close to 0.90 for every adaptive method in every design, so the comparison rests on the regime-level measures and on width.

\begin{table}[p]
\centering
\caption{Simulation results (250 repetitions each). Coverage and width are overall. Regime dev. is the mean over regimes of the absolute deviation of coverage from 0.90, and worst short. is the largest regime shortfall below 0.90; lower is better for both. The last two columns are paired differences from projected ACI with 95\% normal intervals over repetitions. The regime-level measures are post hoc. Spectral AgACI was run in Simulations 1 and 2 only. The online gate (not shown) differs from projected spectral ACI by at most 0.0007 in regime deviation in every design.}
\label{tab:sim}
\tiny
\setlength{\tabcolsep}{2pt}
\resizebox{\linewidth}{!}{\begin{tabular}{lrrrrrr}
\toprule
Method & Coverage & Width & Regime dev. & Worst short. & Dev. minus ACI (95\% CI) & Width minus ACI (95\% CI) \\
\midrule
\multicolumn{7}{l}{\textit{Simulation 1: recurring regimes} (250 repetitions)} \\
Split & 0.8994 & 3.194 & 0.0840 & 0.169 & +0.0498 (+0.0484, +0.0512) & -0.035 (-0.051, -0.018) \\
Projected ACI & 0.9006 & 3.228 & 0.0342 & 0.056 & -- & -- \\
Scaled ACI & 0.8997 & 3.414 & 0.0083 & 0.013 & -0.0259 (-0.0266, -0.0251) & +0.186 (+0.176, +0.195) \\
AgACI & 0.9046 & 3.204 & 0.0483 & 0.085 & +0.0141 (+0.0135, +0.0148) & -0.024 (-0.029, -0.020) \\
Conformal PID & 0.8979 & 3.238 & 0.0134 & 0.027 & -0.0208 (-0.0213, -0.0202) & +0.009 (+0.002, +0.017) \\
Projected spectral ACI & 0.9001 & 3.237 & 0.0276 & 0.052 & -0.0065 (-0.0074, -0.0057) & +0.008 (+0.002, +0.015) \\
Spectral AgACI & 0.9032 & 3.206 & 0.0398 & 0.072 & +0.0056 (+0.0045, +0.0067) & -0.022 (-0.028, -0.016) \\
Oracle regime ACI & 0.9000 & 3.218 & 0.0093 & 0.014 & -0.0249 (-0.0257, -0.0240) & -0.011 (-0.018, -0.004) \\
\midrule
\multicolumn{7}{l}{\textit{Simulation 2: gradual frequency change} (250 repetitions)} \\
Split & 0.9698 & 4.135 & 0.0564 & -0.022 & +0.0454 (+0.0437, +0.0470) & +0.877 (+0.860, +0.893) \\
Projected ACI & 0.8993 & 3.258 & 0.0111 & 0.016 & -- & -- \\
Scaled ACI & 0.9009 & 3.415 & 0.0084 & 0.007 & -0.0027 (-0.0035, -0.0019) & +0.157 (+0.149, +0.166) \\
AgACI & 0.9096 & 3.319 & 0.0201 & 0.024 & +0.0090 (+0.0083, +0.0098) & +0.061 (+0.058, +0.065) \\
Conformal PID & 0.8992 & 3.217 & 0.0073 & 0.011 & -0.0038 (-0.0043, -0.0032) & -0.042 (-0.049, -0.034) \\
Projected spectral ACI & 0.8986 & 3.248 & 0.0105 & 0.010 & -0.0006 (-0.0013, +0.0002) & -0.010 (-0.016, -0.004) \\
Spectral AgACI & 0.9014 & 3.241 & 0.0124 & 0.011 & +0.0013 (+0.0004, +0.0022) & -0.017 (-0.024, -0.011) \\
Oracle regime ACI & 0.8992 & 3.261 & 0.0096 & 0.008 & -0.0014 (-0.0021, -0.0008) & +0.003 (-0.002, +0.007) \\
\midrule
\multicolumn{7}{l}{\textit{Simulation 1b: equal scale, different noise shape} (250 repetitions)} \\
Split & 0.9015 & 2.993 & 0.0389 & 0.049 & +0.0177 (+0.0168, +0.0186) & -0.032 (-0.044, -0.020) \\
Projected ACI & 0.9023 & 3.025 & 0.0211 & 0.027 & -- & -- \\
Scaled ACI & 0.9055 & 3.135 & 0.0271 & 0.035 & +0.0059 (+0.0049, +0.0069) & +0.110 (+0.101, +0.118) \\
AgACI & 0.9049 & 3.021 & 0.0265 & 0.031 & +0.0054 (+0.0049, +0.0058) & -0.004 (-0.007, -0.000) \\
Conformal PID & 0.8994 & 3.173 & 0.0056 & 0.009 & -0.0155 (-0.0163, -0.0147) & +0.148 (+0.135, +0.161) \\
Projected spectral ACI & 0.9014 & 3.037 & 0.0175 & 0.023 & -0.0036 (-0.0041, -0.0031) & +0.012 (+0.008, +0.016) \\
Oracle regime ACI & 0.8999 & 3.003 & 0.0095 & 0.014 & -0.0117 (-0.0127, -0.0107) & -0.022 (-0.028, -0.016) \\
\midrule
\multicolumn{7}{l}{\textit{Simulation 3: exogenous features} (250 repetitions)} \\
Split & 0.8996 & 2.305 & 0.1285 & 0.258 & +0.0631 (+0.0620, +0.0643) & +0.079 (+0.064, +0.094) \\
Projected ACI & 0.9037 & 2.226 & 0.0654 & 0.073 & -- & -- \\
Scaled ACI & 0.9004 & 2.266 & 0.0072 & 0.010 & -0.0582 (-0.0589, -0.0575) & +0.040 (+0.032, +0.048) \\
AgACI & 0.9118 & 2.219 & 0.0640 & 0.069 & -0.0014 (-0.0020, -0.0008) & -0.007 (-0.011, -0.003) \\
Conformal PID & 0.8976 & 2.229 & 0.0229 & 0.034 & -0.0425 (-0.0430, -0.0420) & +0.003 (-0.004, +0.009) \\
Projected spectral ACI & 0.9003 & 2.222 & 0.0162 & 0.023 & -0.0491 (-0.0499, -0.0483) & -0.004 (-0.012, +0.003) \\
Oracle regime ACI & 0.9001 & 2.082 & 0.0093 & 0.014 & -0.0560 (-0.0568, -0.0553) & -0.144 (-0.152, -0.137) \\
\bottomrule
\end{tabular}}
\end{table}

\emph{Simulation 1.} Projected spectral ACI reduces the mean regime deviation from 0.0342 to 0.0276, a difference of \(-0.0065\) (interval \(-0.0074\) to \(-0.0057\)), at a width cost of 0.3\%. This is a statistically clear but small gain: it closes about a quarter of the gap between projected ACI (0.0342) and the oracle (0.0093). Scaled ACI does better than spectral ACI by a wide margin (0.0083 against 0.0276), at the price of 5.7\% wider intervals than ACI, and Conformal PID (0.0134) does better at a width cost of 0.3\% relative to ACI. AgACI, with slightly narrower intervals, is worse than ACI on this measure (0.0483).

\emph{Simulation 2.} Projected spectral ACI and projected ACI do not differ (\(-0.0006\), interval \(-0.0013\) to \(0.0002\)). Scaled ACI, Conformal PID and the oracle all do slightly better than ACI (differences \(-0.0027\), \(-0.0038\) and \(-0.0014\)). Fixed spectral conformal prediction without feedback and the non-adaptive baselines over-cover, as in the first version.

\emph{Simulation 1b.} Projected spectral ACI lowers the regime deviation from 0.0211 to 0.0175 (\(-0.0036\), interval \(-0.0041\) to \(-0.0031\)), a relative reduction of 17\%, similar to the 19\% of Simulation~1. The oracle reaches 0.0095, so spectral weighting closes about 31\% of the gap. Scaled ACI is worse than ACI here (0.0271), as designed. Conformal PID reaches 0.0056, the lowest of the usable methods, with intervals 4.9\% wider than ACI. So even in a design built to favour similarity weights, a feedback method with a fast tracker calibrates more evenly across regimes, at the cost of width.

\emph{Simulation 3.} With exogenous features, projected spectral ACI lowers the regime deviation from 0.0654 to 0.0162 (\(-0.0491\), interval \(-0.0499\) to \(-0.0483\)), which is 88\% of the gain available from the oracle (0.0654 down to 0.0093), at essentially the same width. The worst regime shortfall falls from 0.073 to 0.023. Scaled ACI (0.0072) and Conformal PID (0.0229) also improve on ACI. The oracle's intervals are 6.5\% narrower than ACI's at the same coverage (2.08 against 2.23), which is the efficiency that exact regime knowledge would give. Spectral weighting does not reach it (width 0.2\% below ACI).

Taken together, spectral weighting lowers regime-level miscalibration relative to equal weights in all three designs that have recurring regimes, and not by a clear margin when the frequency drifts gradually. How much it recovers depends on how clean the features are. With exogenous features it recovers most of the oracle's gain. With features computed from the response, windows that cross a regime boundary are mixtures of two spectra, and it recovers about a quarter to a third. In each of the three recurring-regime designs, at least one baseline that does not use spectral information calibrates better than spectral ACI: scaled ACI in Simulations~1 and~3, and Conformal PID in Simulations~1 and~1b. The cost, relative to projected ACI, is wider intervals, by 0.1\% to 5.7\% depending on the design and method (the smallest figure, for Conformal PID in Simulation~3, is not distinguishable from zero).

\subsection{Simulation 3: how loose is the bound?}
\label{sec:sim3}

Simulation~3 satisfies Assumption~\ref{ass:working}. We use it to check the cluster form of the bound, which is the proof of Theorem~\ref{thm:fixed} with the mismatch term bounded by \(W_M\) through Proposition~\ref{prop:cluster}. We do not estimate the Lipschitz constant, so the form with \(LB_t\) is not checked.

The check uses fixed spectral conformal prediction without feedback, a bandwidth fixed at 0.15 before the runs, and the 960 calibration residuals as the pool. Because the response is normal with known scale, the true conditional coverage of the quantile \(\widehat q_t\) is \(2\Phi(\widehat q_t/\sigma_{r})-1\) and can be computed exactly. The deficit is \(0.90\) minus this true coverage. The mismatch weight \(W_M\) is the total weight on pool points from other regimes (Proposition~\ref{prop:cluster}), and \(\varepsilon_n\) is the grid bound of Proposition~\ref{prop:neff-concentration} with \(\delta=0.05\) and a grid of 960 points, so the comparison is with \(W_M+\varepsilon_n\). The grid is the set of calibration residuals, which is chosen after seeing the data, so this check uses the grid constant as a heuristic and is not an instance of Proposition~\ref{prop:neff-concentration}, whose grid is fixed in advance.

\begin{table}[h]
\centering
\caption{Simulation 3: true conditional coverage of fixed spectral conformal prediction (target 0.90) and the quantities in the bound, averaged over test points in each regime and over 250 repetitions. Shares are the fraction of test points where the deficit is no larger than the stated bound.}
\label{tab:sim3-bound}
\tiny
\setlength{\tabcolsep}{3pt}
\resizebox{\linewidth}{!}{\begin{tabular}{lrrrrrrr}
\toprule
Regime & True coverage, spectral & True coverage, equal weights & Mismatch weight $W_M$ & $n_{\mathrm{eff}}$ & $\varepsilon_n$ & Share with deficit $\le W_M+\varepsilon_n$ & Share with deficit $\le W_M$ \\
\midrule
0 & 0.943 & 1.000 & 0.266 & 136.6 & 0.200 & 1.000 & 0.965 \\
1 & 0.882 & 0.963 & 0.286 & 178.4 & 0.176 & 1.000 & 1.000 \\
2 & 0.830 & 0.643 & 0.270 & 164.5 & 0.186 & 1.000 & 0.999 \\
3 & 0.917 & 0.994 & 0.279 & 175.3 & 0.179 & 1.000 & 1.000 \\
\bottomrule
\end{tabular}}
\end{table}

The deficit stayed below \(W_M+\varepsilon_n\) at every test point (Table~\ref{tab:sim3-bound}), and below \(W_M\) alone at 96.5\% to 100\% of them. The bound is loose. The mismatch weight is 0.27 to 0.29 and \(\varepsilon_n\) is 0.18 to 0.20, so \(W_M+\varepsilon_n\) is about 0.46, while the largest average deficit, in the high-noise regime, is 0.07. Equal weights have a mismatch weight of 0.75 and a deficit of 0.26 in that regime (true coverage 64.3\%). Spectral weights cut the worst regime deficit by about three quarters. The theory identifies the right sources of error and the right direction, but its constants are too large to be useful as a stopping rule, and \(L\) is not observable, so the smoothness form of the bound (\(LB_t\)) cannot be computed at all.

\section{Real-data studies}
\label{sec:real}

\subsection{Four single chronological splits}
\label{sec:single}

The first study uses four public series: the monthly U.S. industrial production index for electric and gas utilities \citep{fredIPG2211A2N}, weekly U.S. regular gasoline prices \citep{fredGASREGW}, daily Seattle maximum temperature \citep{vegaSeattleWeather}, and daily German electricity consumption \citep{wiese2019open}. In each series the first 60\% of the record fits the point forecaster, the next 20\% is the calibration block and the last 20\% is the test block. No test observation enters a feature or a forecast made earlier.

The forecaster is a linear regression with a trend, seasonal Fourier terms and lagged values (one and twelve months for utilities; one, four and fifty-two weeks for gasoline; one day, seven days and one year for Seattle; one and seven days with day of week terms for Germany). It is deliberately simple. We want to study the calibration layer, and a strong forecaster would remove structure from the residuals that the calibration methods have to handle. Section~\ref{sec:limits} returns to this.

Table~\ref{tab:single} gives coverage with a moving block bootstrap interval, and the paired differences from projected ACI in coverage and width. The test blocks hold 208, 363, 220 and 876 points.

\begin{table}[h]
\centering
\caption{Single chronological splits. Target coverage is 0.90. Intervals are moving block bootstrap intervals (blocks of 20 test points, 2000 resamples); paired differences use the same resamples for both methods. The safeguard row is the rule as coded (calibration leave-one-out median effective sample size floor of 20); it selects the same bandwidth as the unsafeguarded method on all four series. The gated row is exploratory.}
\label{tab:single}
\tiny
\setlength{\tabcolsep}{2pt}
\resizebox{\linewidth}{!}{\begin{tabular}{llrrrr}
\toprule
Series & Method & Coverage (95\% block bootstrap) & Average width & Coverage minus ACI (95\%) & Width minus ACI (95\%) \\
\midrule
U.S. utilities (208 test points) & Split & 0.827 (0.784, 0.870) & 10.89 & -0.058 (-0.096, -0.024) & -3.20 (-4.50, -2.03) \\
 & Projected ACI & 0.885 (0.851, 0.918) & 14.08 & -- & -- \\
 & Scaled ACI & 0.894 (0.861, 0.928) & 15.59 & +0.010 (-0.034, +0.053) & +1.51 (+0.22, +3.03) \\
 & Conformal PID & 0.889 (0.856, 0.923) & 15.24 & +0.005 (-0.024, +0.034) & +1.15 (-0.30, +2.62) \\
 & Conformal PID + scorecaster & 0.894 (0.861, 0.923) & 15.04 & +0.010 (-0.024, +0.038) & +0.96 (-0.59, +2.82) \\
 & Projected Spectral ACI & 0.784 (0.726, 0.846) & 12.93 & -0.101 (-0.159, -0.034) & -1.15 (-3.09, +1.38) \\
 & Safe Projected Spectral ACI & 0.784 (0.726, 0.846) & 12.93 & -0.101 (-0.159, -0.034) & -1.15 (-3.09, +1.38) \\
 & Gated Spectral ACI & 0.889 (0.861, 0.928) & 14.80 & +0.005 (-0.019, +0.034) & +0.72 (-0.60, +2.92) \\
\midrule
U.S. gasoline (363 test points) & Split & 0.562 (0.479, 0.653) & 0.209 & -0.320 (-0.405, -0.237) & -0.097 (-0.120, -0.070) \\
 & Projected ACI & 0.882 (0.851, 0.920) & 0.305 & -- & -- \\
 & Scaled ACI & 0.909 (0.873, 0.945) & 0.318 & +0.028 (+0.000, +0.050) & +0.013 (-0.000, +0.027) \\
 & Conformal PID & 0.884 (0.865, 0.912) & 0.314 & +0.003 (-0.019, +0.025) & +0.008 (-0.001, +0.019) \\
 & Conformal PID + scorecaster & 0.890 (0.873, 0.909) & 0.331 & +0.008 (-0.022, +0.030) & +0.025 (+0.013, +0.038) \\
 & Projected Spectral ACI & 0.895 (0.860, 0.931) & 0.320 & +0.014 (-0.019, +0.036) & +0.015 (-0.002, +0.030) \\
 & Safe Projected Spectral ACI & 0.895 (0.860, 0.931) & 0.320 & +0.014 (-0.019, +0.036) & +0.015 (-0.002, +0.030) \\
 & Gated Spectral ACI & 0.895 (0.860, 0.931) & 0.320 & +0.014 (-0.017, +0.033) & +0.014 (+0.001, +0.024) \\
\midrule
Seattle temperature (220 test points) & Split & 0.841 (0.782, 0.877) & 7.77 & -0.064 (-0.118, -0.036) & -1.37 (-1.90, -1.10) \\
 & Projected ACI & 0.905 (0.877, 0.936) & 9.13 & -- & -- \\
 & Scaled ACI & 0.905 (0.868, 0.932) & 9.68 & +0.000 (-0.023, +0.014) & +0.55 (-0.17, +0.95) \\
 & Conformal PID & 0.900 (0.882, 0.923) & 9.90 & -0.005 (-0.023, +0.018) & +0.76 (+0.06, +1.35) \\
 & Conformal PID + scorecaster & 0.895 (0.873, 0.923) & 10.14 & -0.009 (-0.027, +0.014) & +1.01 (+0.28, +1.60) \\
 & Projected Spectral ACI & 0.905 (0.873, 0.932) & 9.00 & +0.000 (-0.014, +0.005) & -0.13 (-0.36, +0.05) \\
 & Safe Projected Spectral ACI & 0.905 (0.873, 0.932) & 9.00 & +0.000 (-0.014, +0.005) & -0.13 (-0.36, +0.05) \\
 & Gated Spectral ACI & 0.905 (0.873, 0.932) & 9.00 & +0.000 (-0.014, +0.005) & -0.13 (-0.36, +0.05) \\
\midrule
German electricity (876 test points) & Split & 0.886 (0.853, 0.919) & 227.2 & -0.011 (-0.037, +0.011) & -27.7 (-50.5, -6.5) \\
 & Projected ACI & 0.897 (0.874, 0.921) & 254.9 & -- & -- \\
 & Scaled ACI & 0.900 (0.881, 0.920) & 299.7 & +0.002 (-0.016, +0.021) & +44.8 (+21.2, +71.9) \\
 & Conformal PID & 0.896 (0.878, 0.917) & 271.6 & -0.001 (-0.015, +0.013) & +16.7 (+4.2, +28.5) \\
 & Conformal PID + scorecaster & 0.896 (0.882, 0.913) & 212.6 & -0.001 (-0.021, +0.021) & -42.3 (-64.6, -23.2) \\
 & Projected Spectral ACI & 0.897 (0.876, 0.920) & 251.1 & +0.000 (-0.007, +0.007) & -3.8 (-10.4, +1.6) \\
 & Safe Projected Spectral ACI & 0.897 (0.876, 0.920) & 251.1 & +0.000 (-0.007, +0.007) & -3.8 (-10.4, +1.6) \\
 & Gated Spectral ACI & 0.897 (0.876, 0.920) & 251.1 & +0.000 (-0.007, +0.007) & -3.8 (-10.4, +1.6) \\
\bottomrule
\end{tabular}}
\end{table}

Four results stand out.

First, split conformal prediction fails on three of the four series. Coverage is 56.2\% on gasoline, 82.7\% on utilities and 84.1\% on Seattle, and all three intervals for the difference from ACI exclude zero. Adaptive feedback matters more here than anything else in the paper: projected ACI, scaled ACI and Conformal PID all reach between 88\% and 91\% on every series.

Second, spectral ACI is not distinguishable from projected ACI on gasoline, Seattle or Germany. Its coverage is 1.4 points higher on gasoline (interval \(-1.9\) to \(+3.6\) points, 363 test points) and identical on Seattle and Germany. Its average intervals are 1.4\% and 1.5\% narrower on Seattle and Germany, but the intervals for the width difference include zero. The earlier version of this paper called these two ``slightly shorter''. We withdraw that as a finding.

Third, spectral ACI fails on utilities. Coverage is 78.4\%, which is 10.1 points below projected ACI (interval \(-15.9\) to \(-3.4\)). The cause is a small effective sample size (Section~\ref{sec:safeguard}). The safeguard as coded does not change this result.

Fourth, Conformal PID with the Theta scorecaster is the only method that gives clearly narrower intervals at the same coverage, and only on one series. On German electricity it covers 89.6\% against 89.7\% for ACI, with an average width of 213 against 255 (difference \(-42\), interval \(-65\) to \(-23\)). On the other three series the scorecaster does not help, and without it PID is wider than ACI on Germany.

\subsection{Rolling-origin study on nine series}
\label{sec:rolling}

A single split depends heavily on one date, so we repeat the comparison at four forecast origins (55\%, 65\%, 75\% and 85\% of each record) for nine public series: the four above, Mauna Loa carbon dioxide, sunspots, El Ni\~no, U.S. real GDP and Nile flow. The last five are distributed with \texttt{statsmodels} \citep{seabold2010statsmodels}. This gives 36 test blocks. The forecaster is refitted before each calibration block, which is 20\% of the record, and each test block is 10\%. It is a least squares regression on the training part with an intercept, a linear trend, sine and cosine terms at the nominal period and at half the nominal period, and the observations one step and one period earlier. The nominal periods are 12, 52, 365, 7, 52, 11, 12, 4 and 10 steps, in the order of the window lengths in Section~\ref{sec:features}. The five original methods are reproduced exactly from the saved results of the first version (largest difference 0 over 180 rows). Four methods were added (Section~\ref{sec:design}), together with scaled ACI.

\begin{table}[h]
\centering
\caption{Rolling-origin results, 36 origins. The last column is the mean absolute deviation from 0.90 of coverage in seasonal phase bins, over 32 origins (post hoc measure, Section~\ref{sec:design}). Lower is better in the last three columns. The last two rows are a post hoc check with a finite-sample rank in the weighted quantile (the phase-bin measure was computed for the projected variant only).}
\label{tab:rolling}
\scriptsize
\setlength{\tabcolsep}{3pt}
\resizebox{\linewidth}{!}{\begin{tabular}{lrrrrr}
\toprule
Method & Mean coverage & Mean abs. coverage error & Normalized width & Normalized interval score & Mean abs. phase-bin deviation \\
\midrule
Split & 0.761 & 0.158 & 4.82 & 9.02 & 0.173 \\
Projected ACI & 0.823 & 0.090 & 5.21 & 8.06 & 0.108 \\
Scaled ACI & 0.912 & 0.028 & 6.33 & 7.89 & 0.070 \\
AgACI-EWA & 0.842 & 0.069 & 5.38 & 7.86 & 0.096 \\
Conformal PID & 0.882 & 0.025 & 6.26 & 8.31 & 0.061 \\
Conformal PID + scorecaster & 0.884 & 0.017 & 6.44 & 8.34 & 0.053 \\
Projected Spectral ACI & 0.805 & 0.099 & 5.18 & 8.35 & 0.109 \\
Safe Projected Spectral ACI & 0.821 & 0.089 & 5.25 & 8.24 & 0.108 \\
Spectral AgACI-EWA & 0.821 & 0.085 & 5.33 & 8.14 & 0.103 \\
Gated Spectral ACI & 0.825 & 0.086 & 5.26 & 8.20 & 0.106 \\
\midrule
\multicolumn{6}{l}{\textit{Post hoc check: finite-sample rank in the weighted quantile}} \\
Projected Spectral ACI (rank corrected) & 0.864 & 0.057 & 5.98 & 8.20 & 0.087 \\
Safe Projected Spectral ACI (rank corrected) & 0.854 & 0.067 & 5.88 & 8.25 & -- \\
\bottomrule
\end{tabular}}
\end{table}

Table~\ref{tab:rolling} shows three groups. Split and projected ACI under-cover on average (76.1\% and 82.3\%), and so do AgACI (84.2\%) and all the spectral variants (80.5\% to 82.5\%). Scaled ACI, Conformal PID and PID with the scorecaster come close to the target (91.2\%, 88.2\% and 88.4\%) with intervals about 20\% wider in normalized width. The best interval score belongs to AgACI (7.86), followed closely by scaled ACI (7.89).

Table~\ref{tab:paired-aci} compares each method with projected ACI. With the plain weighted quantile, the spectral variants differ from ACI in one respect only: their normalized interval scores are worse. Projected spectral ACI is worse by 0.30 (series bootstrap interval 0.04 to 0.74; exact sign-flip \(p=0.016\); worse on 7 of 9 series), the safeguarded version by 0.19 (0.03 to 0.37; \(p=0.016\)) and the gated version by 0.14 (0.01 to 0.32; \(p=0.016\)). Their absolute coverage errors do not differ from ACI (sign-flip \(p\) between 0.39 and 0.87), and neither does their phase-bin measure of conditional calibration, which differs from ACI by at most 0.002 for the projected, safeguarded and gated variants (post hoc measure). Spectral AgACI is not clearly different from ACI on either measure, but it is worse than AgACI on normalized interval score on all nine series (difference 0.29, series bootstrap interval 0.04 to 0.71; sign-flip \(p=0.004\); file \path{results/rolling_v2_paired.csv} in the release).

\begin{table}[h]
\centering
\caption{Paired differences from projected ACI over 36 origins. Positive values are worse for both metrics. Intervals are indicative: the origin bootstrap treats origins as independent, and the series bootstrap uses only nine clusters. The last column gives the exact sign-flip $p$-value over the nine series means, then the number of series where the method is worse and better.}
\label{tab:paired-aci}
\tiny
\setlength{\tabcolsep}{2pt}
\resizebox{\linewidth}{!}{\begin{tabular}{llrrrr}
\toprule
Method minus Projected ACI & Metric & Mean difference & Origin bootstrap 95\% & Series bootstrap 95\% & Series sign-flip $p$ (series worse / better) \\
\midrule
Split & Abs. coverage error & +0.068 & (+0.040, +0.102) & (+0.023, +0.131) & 0.008 (8 / 0) \\
Split & Norm. interval score & +0.962 & (+0.352, +1.858) & (+0.182, +2.573) & 0.008 (8 / 1) \\
Scaled ACI & Abs. coverage error & -0.062 & (-0.100, -0.029) & (-0.121, -0.016) & 0.020 (1 / 8) \\
Scaled ACI & Norm. interval score & -0.166 & (-0.566, +0.178) & (-0.724, +0.262) & 0.441 (5 / 4) \\
AgACI-EWA & Abs. coverage error & -0.021 & (-0.031, -0.011) & (-0.037, -0.007) & 0.008 (1 / 8) \\
AgACI-EWA & Norm. interval score & -0.198 & (-0.314, -0.096) & (-0.387, -0.057) & 0.008 (1 / 8) \\
Conformal PID & Abs. coverage error & -0.065 & (-0.096, -0.038) & (-0.115, -0.026) & 0.004 (0 / 9) \\
Conformal PID & Norm. interval score & +0.253 & (-0.067, +0.563) & (-0.196, +0.672) & 0.223 (6 / 3) \\
Conformal PID + scorecaster & Abs. coverage error & -0.073 & (-0.109, -0.042) & (-0.133, -0.027) & 0.004 (0 / 9) \\
Conformal PID + scorecaster & Norm. interval score & +0.279 & (-0.236, +0.750) & (-0.434, +0.903) & 0.355 (7 / 2) \\
Projected Spectral ACI & Abs. coverage error & +0.010 & (-0.008, +0.032) & (-0.010, +0.041) & 0.426 (4 / 5) \\
Projected Spectral ACI & Norm. interval score & +0.298 & (+0.059, +0.584) & (+0.040, +0.736) & 0.016 (7 / 2) \\
Safe Projected Spectral ACI & Abs. coverage error & -0.000 & (-0.008, +0.008) & (-0.011, +0.010) & 0.867 (3 / 5) \\
Safe Projected Spectral ACI & Norm. interval score & +0.187 & (+0.056, +0.336) & (+0.031, +0.371) & 0.016 (7 / 2) \\
Spectral AgACI-EWA & Abs. coverage error & -0.005 & (-0.023, +0.018) & (-0.028, +0.025) & 0.641 (3 / 6) \\
Spectral AgACI-EWA & Norm. interval score & +0.087 & (-0.201, +0.403) & (-0.279, +0.569) & 0.754 (5 / 4) \\
Gated Spectral ACI & Abs. coverage error & -0.004 & (-0.011, +0.003) & (-0.014, +0.005) & 0.391 (3 / 5) \\
Gated Spectral ACI & Norm. interval score & +0.139 & (+0.033, +0.266) & (+0.007, +0.321) & 0.016 (7 / 2) \\
\midrule
\multicolumn{6}{l}{\textit{Post hoc check: finite-sample rank in the weighted quantile}} \\
Projected Spectral ACI (rank corrected) & Abs. coverage error & -0.033 & (-0.054, -0.014) & (-0.071, -0.005) & 0.023 (1 / 7) \\
Projected Spectral ACI (rank corrected) & Norm. interval score & +0.141 & (-0.096, +0.387) & (-0.265, +0.550) & 0.523 (5 / 4) \\
Safe Projected Spectral ACI (rank corrected) & Abs. coverage error & -0.023 & (-0.040, -0.008) & (-0.055, -0.002) & 0.047 (1 / 6) \\
Safe Projected Spectral ACI (rank corrected) & Norm. interval score & +0.192 & (-0.028, +0.429) & (-0.198, +0.589) & 0.422 (5 / 3) \\
\bottomrule
\end{tabular}}
\end{table}

The question that matters for the method is whether spectral weights add anything to a plain error-scale tracker. They do not. Against scaled ACI (Table~S6 in the supplement), the four spectral variants have higher absolute coverage error by 0.06 to 0.07 (sign-flip \(p\) between 0.02 and 0.04) and a worse normalized interval score by 0.25 to 0.46, although the series bootstrap intervals for the interval score all include zero.

Conformal PID is better calibrated than ACI on all nine series (absolute coverage error lower by 0.065, sign-flip \(p=0.004\)). Its interval score is not clearly different from ACI's (worse by 0.25, series bootstrap interval \(-0.20\) to \(0.67\)) because its intervals are wider. Whether that trade is wanted depends on the cost of under-coverage.

\emph{Quantile convention.} Projected ACI uses the order statistic of rank \(\lceil(N+1)(1-a)\rceil\). The plain weighted quantile of the spectral methods uses a rank about one lower, which matters when the pool is small. On Nile flow the safeguard falls back to equal weights at all four origins, so the difference between the safeguarded method and ACI comes only from this rule (normalized interval score 8.65 against 8.19; supplement Table~S3). A review pointed this out. We added a post hoc check, shown in the last two rows of Table~\ref{tab:rolling} and Table~\ref{tab:paired-aci}, in which the spectral quantile has the finite-sample form of Corollary~\ref{cor:barber} and equals ACI exactly for equal weights.

The check changes the coverage result. With the corrected rank, the projected and safeguarded spectral variants cover 86.4\% and 85.4\% on average, against 82.3\% for ACI. Their mean absolute coverage error is 0.057 and 0.067 against 0.090, a difference of \(-0.033\) (sign-flip \(p=0.023\)) and \(-0.023\) (\(p=0.047\)). Their intervals are 13\% to 15\% wider in normalized width (5.98 and 5.88 against 5.21), and their interval scores do not differ clearly from ACI's (\(+0.14\) and \(+0.19\); sign-flip \(p=0.52\) and \(0.42\)). They remain worse calibrated than scaled ACI (absolute coverage error higher by 0.030 and 0.039, \(p=0.031\)) and Conformal PID (higher by 0.033 and 0.043, \(p=0.004\)).

The corrected rule is more conservative at low kernel mass, and it is more than a one-rank shift. Its target level exceeds one, and the code then returns the largest pool residual. This happens at 5.6\% of test steps overall for the projected variant (4.7\% for the safeguarded one), at every step on Nile flow and at 27\% to 29\% of steps on average on sunspots and utilities, against 0.0\% of steps for projected ACI. The median kernel mass is below 30 at 15 of the 36 origins. The gain does not come from this cap alone. On the 24 origins where the cap applies to under 10\% of steps, the corrected projected variant has mean absolute coverage error 0.057 against 0.082 for ACI. On the six long series the coverage error improves by only 0.010 (\(p=0.19\)), and the interval score is higher by 0.23 (\(p=0.19\)).

So the finding that spectral weights do not change coverage depends on the quantile convention. The rest does not: with either rule they are no better than ACI on interval score, and no better calibrated than scaled ACI or Conformal PID.

\emph{Short test blocks.} Three series have short test blocks (10 to 31 points, Section~\ref{sec:design}), and ACI-type methods barely move within them at \(\gamma=0.005\). Restricting to the six series with at least 40 test points (24 origins, supplement Table~S4) keeps the direction but weakens the evidence. The projected, safeguarded and gated variants differ from ACI in coverage error by at most 0.003, and their normalized interval scores are higher by 0.07 to 0.10. The six-series sign-flip test cannot confirm this: its smallest attainable \(p\) is 0.031, and we observe \(p=0.125\) for each variant, which is worse on 4 of 6 series. Removing only the four Nile flow origins leaves a score difference of 0.16 for the projected variant (\(p=0.031\)). On the six series PID has absolute coverage error 0.008 against 0.053 for ACI. With a learning rate of 0.02 (file \path{results/rolling_v2_gamma002_overall.csv} in the release) the ACI family improves and the spectral variants are still no better than ACI: projected spectral ACI has absolute coverage error 0.072 against 0.058 for ACI, and a worse interval score (8.15 against 7.92). Scaled ACI has mean absolute coverage error 0.021.

\emph{Difficult series.} On U.S. real GDP (20 test points per origin) and Seattle, split, ACI, AgACI and the spectral variants fall far short of the target. Mean coverage per origin is between 37.5\% and 70.0\% on GDP and between 69.0\% and 77.1\% on Seattle. Scaled ACI and Conformal PID come much closer (GDP 82.5\% to 92.5\%, Seattle 88.7\% to 94.3\%). In these series the size of the errors changes, and a method that follows the error scale handles that better than one that only moves a single level. We keep these series in the averages.

On nine public series, then, spectral weighting with the plain weighted quantile did not change coverage or conditional calibration relative to ACI and raised the mean interval score. With the finite-sample quantile it covered closer to the target than ACI, at wider intervals and with no better interval score. In neither version did it match scaled ACI or Conformal PID in calibration. We cannot say it never helps, because the forecasters are simple and nine series with short blocks give wide intervals. We can say that these data do not show that it is the better choice.

\section{Effective sample size and the bandwidth safeguard}
\label{sec:safeguard}

\subsection{What the effective sample size shows}

Table~\ref{tab:ess-diagnostics} gives the distribution of the effective sample size at the test points for the spectral methods in the simulations and the four single splits. A small value means the quantile rests on a few residuals. The utilities series is the extreme case, with a median of 3.7 for fixed spectral weighting and 3.3 for spectral ACI, and it is the series on which spectral ACI fails (Section~\ref{sec:single}).

\begin{table}[h]
\centering
\caption{Effective sample size at the test points (values from the first version of the study). Tail values matter, because weighted quantiles become unstable when few calibration residuals carry weight.}
\label{tab:ess-diagnostics}
\scriptsize
\begin{tabular}{llrrrr}
\toprule
Experiment & Method & Mean & 10th pct. & Median & 90th pct. \\
\midrule
Simulation 1 & Spectral & 351.99 & 102.69 & 311.29 & 688.32 \\
Simulation 1 & Spectral ACI & 517.80 & 132.06 & 442.95 & 1008.53 \\
Simulation 2 & Spectral & 151.58 & 3.55 & 62.88 & 452.11 \\
Simulation 2 & Spectral ACI & 287.97 & 25.43 & 249.72 & 623.20 \\
U.S. utilities & Spectral & 8.07 & 2.74 & 3.70 & 17.98 \\
U.S. utilities & Spectral ACI & 7.09 & 1.09 & 3.33 & 10.79 \\
U.S. gasoline & Spectral & 46.30 & 16.11 & 41.13 & 86.51 \\
U.S. gasoline & Spectral ACI & 62.22 & 17.49 & 53.82 & 115.87 \\
Seattle temperature & Spectral & 212.67 & 210.88 & 213.29 & 213.96 \\
Seattle temperature & Spectral ACI & 321.06 & 232.76 & 321.33 & 409.20 \\
German electricity & Spectral & 705.02 & 579.24 & 734.40 & 805.26 \\
German electricity & Spectral ACI & 1052.89 & 712.27 & 1066.58 & 1457.00 \\
\bottomrule
\end{tabular}
\end{table}

The first version of this paper stated that the ordering of the four real series by effective sample size exactly matches their ordering by spectral ACI performance, and called this the pattern the theory predicts. We withdraw that statement. It rests on four points, and the ordering by performance is not defined (Seattle and Germany tie with ACI). The effective sample size also depends on the pool size (207, 362, 219 and 875 calibration points) and on the bandwidth, so it cannot be compared across series as a single quantity. The theory of Section~\ref{sec:theory} predicts a sampling term that falls as \(n_{\mathrm{eff}}\) rises. It predicts no ordering of performance across series.

\subsection{On real series the weights act like a recency kernel}
\label{sec:lag}

Section~\ref{sec:features} noted that the features cannot record the phase within a season. Figure~\ref{fig:lag} shows what that means in practice. For each series it plots the median distance between the feature at time \(t\) and the feature at time \(t-L\), relative to the typical distance between two random times, against the lag in periods. Supplement Table~S10 gives the values at one step, half a period, one period and two periods. In seven of the nine series the distance rises with the lag and does not dip near whole periods. Mauna Loa carbon dioxide does dip near one and two periods, so its features keep some seasonal memory. Seattle falls back after about three quarters of a period, but its distances are tiny in absolute terms (at most 0.025). A noisy sinusoid with a trend also gives a rising curve (supplement Table~S10).

The consequence for the weights depends on the bandwidth. On the utilities series at the validated bandwidth \(b=0.04\), a distance of 0.146 (one year) gets a kernel weight of \(\exp\{-0.146^2/(2\cdot0.04^2)\}=\exp(-6.7)\), under 0.2\% of the weight at distance zero. The weights therefore concentrate on time-adjacent points, whose windows overlap most with the current one. In Seattle the distances are far below the bandwidth, so the weights are nearly equal (median effective sample size 213 for a pool of 219).

\begin{figure}[t]
\centering
\includegraphics[width=\linewidth]{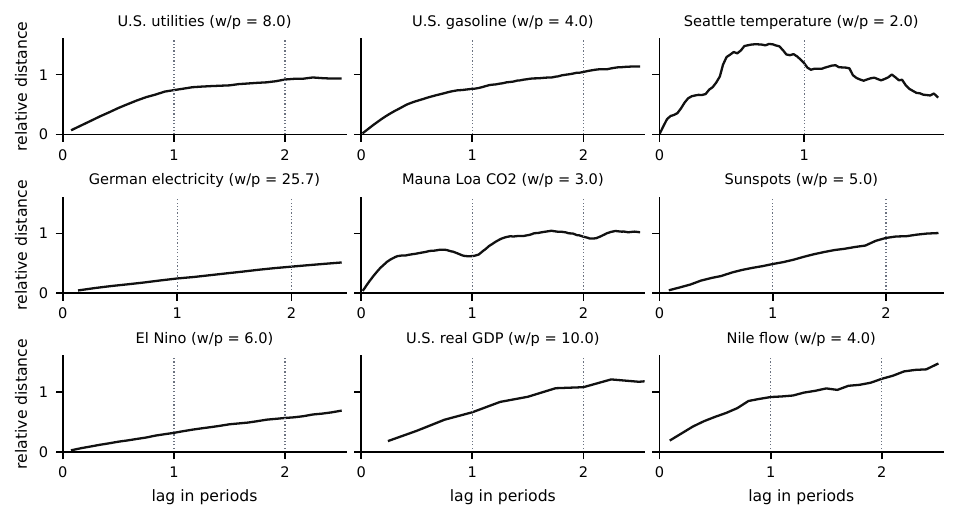}
\caption{Median distance between the spectral feature at time \(t\) and at time \(t-L\), divided by the typical distance between two random times, against the lag in periods. The label w/p gives the window length divided by the nominal period. Dotted lines mark one and two periods. A feature that matched seasons would dip at the dotted lines. Seattle has too few times with full windows beyond 1.9 periods.}
\label{fig:lag}
\end{figure}

This explains much of the empirical pattern. Spectral ACI on these series behaves like projected ACI run on a short, data-dependent recency window, with a low effective sample size when the bandwidth is small. The simulations are not affected, because there the regimes differ in oscillation period, which the features do see. It also means that ``recurring'' in the practical advice below has to mean recurrence of spectral shape, such as a regime with a different period or amplitude balance, and not recurrence of the calendar season. A feature that records phase, such as the position within the period, would be the natural object if season matching were the aim. Section~\ref{sec:phase} tests it, together with a pure recency kernel.

\subsection{The safeguard as coded does not change the utilities result}

The calibration safeguard (Section~\ref{sec:bandwidth-rules}) raises the bandwidth when the leave-one-out median effective sample size over the calibration block is below 20. On utilities at the validated bandwidth \(b=0.04\) that median is 25.7 (supplement Table~S11), above the floor, so the rule keeps \(b=0.04\) and the result is the one in Table~\ref{tab:single}: coverage 78.4\%. On the other three single splits the rule also keeps the validated bandwidth.

The first version said that the safeguard moves utilities to bandwidth 0.12 and that coverage rises from 78.4\% to 88.9\%. That number came from a sensitivity run at \(b=0.12\), chosen after looking at test coverage, and the effective sample size quoted for it (a test-period median of 3.3) is not available at prediction time. It showed that a wider bandwidth would have helped, not that the rule picks it. The statement that the safeguard repairs this case was wrong, and we remove it.

The gap between the calibration median of 25.7 and the test median of 3.3 at the same bandwidth is large. Section~\ref{sec:lag} gives the likely reason. Calibration points have close neighbours in time, whose windows overlap with theirs and whose features are nearest to theirs, and these neighbours raise their leave-one-out effective sample size. A test point, once the pool holds only earlier points, has fewer close neighbours, and its window may have a different low-frequency shape from those in the calibration block. We have not isolated the two effects.

In the rolling-origin study the rule changes the bandwidth at 6 of 36 origins: all four Nile flow origins (to equal weights, since no grid value reaches the floor on 20 calibration points), one sunspots origin and one GDP origin. It never acts on the six longer series. The difference between the safeguarded and the unsafeguarded method in Table~\ref{tab:rolling} therefore comes almost entirely from series with very short blocks.

\subsection{An online gate, added after the failure was seen}

The effective sample size of a test point against the current pool is known before the response is. The gated method (Section~\ref{sec:bandwidth-rules}) uses it at every test time and raises the bandwidth until it reaches 20. On utilities the gate raises the bandwidth to 0.12 (median over test points) and coverage becomes 88.9\% (interval 86.1\% to 92.8\%), with intervals 5.1\% wider than ACI, close to the 88.5\% of ACI. On the other three single splits its coverage equals that of plain spectral ACI, and its average width differs by at most 0.0005 (gasoline). In the nine-series study it does not beat ACI in coverage error (difference \(-0.004\), series bootstrap interval \(-0.014\) to \(0.005\)) and it has a worse interval score (\(+0.14\), interval 0.01 to 0.32).

We added the gate after seeing the utilities result, so this repair of the one failure that motivated it is not evidence that it generalises. What can be said is narrower: a rule that uses the effective sample size at the test point does what the effective sample size is supposed to do on the series where it matters, and it has no effect where it does not. Whether that helps on a new series is untested.

\section{Phase-aware weights and a recency control}
\label{sec:phase}

The spectral features cannot match seasons, and on the real series they behave like a recency kernel (Section~\ref{sec:lag}). We therefore ran two post hoc experiments, added after review: weights on calendar phase, and a pure recency control. Neither was planned in the first version, and we report every result, including those against the idea.

\subsection{Methods}

\emph{Phase ACI} weights residuals by the position within the nominal period \(p\) of the series, through the pair \((\cos 2\pi i/p,\ \sin 2\pi i/p)\) and a Gaussian kernel. The position is a function of the time index alone. It does not depend on any response, so part (a) of Assumption~\ref{ass:working} holds for it, and it is known at prediction time. The period is 12 for utilities and El Ni\~no, 4 for GDP, 11 for sunspots, 10 for Nile flow (which has no seasonal cycle), 7 for German electricity (the day of the week), and the year length in steps (365.25/7 and 365.25) for the weekly and daily series with an annual cycle. In the simulations we use 24, the period of the forecaster, as the analogue of a nominal period. In Simulations 1, 1b and 3 we also use 480, the period of the regime pattern, which stands for a recurrence period that is known.

\emph{Spectral + phase ACI} uses the product of the spectral and phase kernels, each with its own bandwidth. \emph{Recency ACI} is a control: its weights depend only on the time index, scaled by the calibration length, so a residual matters only through how long ago it was observed. Bandwidths come from the validation rule of Section~\ref{sec:bandwidth-rules}. Each method is run with the plain weighted quantile and with the finite-sample rank of Corollary~\ref{cor:barber}, so that spectral, phase and recency weights are compared under the same two conventions.

\subsection{Real series}

\begin{table}[h]
\centering
\caption{Rolling-origin results for phase-aware and recency weights, 36 origins. ``Rank corrected'' uses the finite-sample quantile. The last column is the phase-bin deviation (32 origins). All rows with phase, product or recency weights are post hoc.}
\label{tab:phase-rolling}
\scriptsize
\setlength{\tabcolsep}{3pt}
\resizebox{\linewidth}{!}{\begin{tabular}{lrrrrr}
\toprule
Method & Mean coverage & Mean abs. coverage error & Normalized width & Normalized interval score & Phase-bin deviation \\
\midrule
Projected ACI & 0.823 & 0.090 & 5.21 & 8.06 & 0.108 \\
Scaled ACI & 0.912 & 0.028 & 6.33 & 7.89 & 0.070 \\
Conformal PID & 0.882 & 0.025 & 6.26 & 8.31 & 0.061 \\
Conformal PID + scorecaster & 0.884 & 0.017 & 6.44 & 8.34 & 0.053 \\
Projected Spectral ACI & 0.805 & 0.099 & 5.18 & 8.35 & 0.109 \\
Projected Spectral ACI (rank corrected) & 0.864 & 0.057 & 5.98 & 8.20 & 0.087 \\
Phase ACI & 0.783 & 0.125 & 5.04 & 8.38 & 0.113 \\
Phase ACI (rank corrected) & 0.859 & 0.067 & 5.93 & 8.13 & 0.090 \\
Spectral + phase ACI & 0.758 & 0.147 & 4.93 & 8.95 & 0.127 \\
Spectral + phase ACI (rank corrected) & 0.909 & 0.060 & 7.48 & 8.85 & 0.073 \\
Recency ACI & 0.840 & 0.070 & 5.39 & 8.01 & 0.085 \\
Recency ACI (rank corrected) & 0.886 & 0.039 & 6.06 & 7.88 & 0.072 \\
\bottomrule
\end{tabular}}
\end{table}

With the plain quantile, no spectral, phase or product variant improves on projected ACI (Table~\ref{tab:phase-rolling} and Figure~\ref{fig:forest}). With the finite-sample rank, all four kinds of weight lower the mean coverage error relative to ACI. The difference is clear (sign-flip \(p\le0.05\)) for spectral, phase and recency weights, and not for the product kernel.

The recency control does best. Its mean absolute coverage error is 0.039, against 0.090 for ACI, 0.057 for spectral ACI with the rank correction and 0.067 for phase ACI with the rank correction. It is closer to the target than ACI on all nine series (sign-flip \(p=0.004\)), and its interval score (7.88) is not worse than ACI's (8.06). Against spectral ACI with the rank correction it has lower coverage error (difference \(-0.018\), \(p=0.008\), better on 8 of 9 series) and a lower interval score (\(-0.32\), \(p=0.039\)). Phase ACI with the rank correction cannot be told apart from spectral ACI with the rank correction (coverage error \(+0.010\), \(p=0.66\); interval score \(-0.065\), \(p=0.78\)). Scaled ACI (0.028), Conformal PID (0.025) and Conformal PID with the scorecaster (0.017) remain the best calibrated, and recency with the rank correction is not distinguishable from scaled ACI (\(+0.012\), \(p=0.26\)).

Part of the recency result reflects the conservatism of the corrected rule at low kernel mass (Section~\ref{sec:rolling}). The target level exceeds one at 5.6\% of test steps for spectral weights, 3.8\% for phase, 3.0\% for recency and 25\% for the product kernel, and the median kernel mass is below 30 at 15, 18, 15 and 28 of the 36 origins. The product kernel has a median mass of 11.5 and the worst interval score (8.85), and its coverage of 90.9\% is largely that cap. This is also why its coverage error (0.060) is below that of phase ACI (0.067). The recency result survives when the cap matters less: on the six long series, recency with the rank correction still has lower coverage error than ACI (\(-0.027\), \(p=0.031\)) and than spectral ACI with the rank correction (\(-0.017\), \(p=0.031\)). So the better calibration of the corrected spectral variants in Section~\ref{sec:rolling} comes from the finite-sample rule and from recency, not from spectral similarity. We report about a dozen comparisons with ACI without correction. The recency result (\(p=0.004\)) survives a Bonferroni correction over twelve comparisons (0.047), and the others do not.

\begin{figure}[t]
\centering
\includegraphics[width=\linewidth]{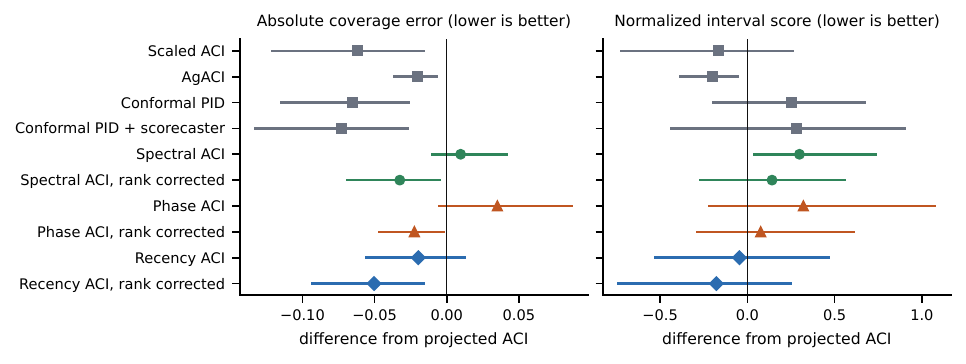}
\caption{Rolling-origin study: difference from projected ACI in absolute coverage error and normalized interval score (36 origins). Points are mean differences, lines are 95\% series bootstrap intervals. Negative values are better. Grey: baselines. Green: spectral weights. Orange: phase weights. Blue: recency weights. Supplement Table~S7 gives the sign-flip $p$-values.}
\label{fig:forest}
\end{figure}

On the six series with at least 40 test points, which are the series with a calendar seasonal cycle, phase ACI is not distinguishable from ACI in coverage error (\(-0.005\), \(p=0.94\)) or interval score (\(-0.063\), \(p=0.59\)). With the rank correction the differences are \(-0.013\) (\(p=0.16\)) and \(-0.085\) (\(p=0.56\)). Nile flow has no seasonal cycle, and U.S. real GDP has no quarterly pattern that the weights can use. Phase weights do worse on both (interval scores 10.60 and 11.49 against 8.19 and 10.34 for ACI). On GDP the validation rule chose the smallest bandwidth at all four origins, which keeps only points from the same quarter, about a quarter of a calibration block of 41 points.

In the single splits (supplement Table~S9), phase ACI repairs the utilities failure: coverage is 89.4\% (interval 86.1\% to 92.3\%) against 78.4\% for spectral ACI, with intervals 1.6\% narrower than ACI's. It also gives much narrower intervals on German electricity at the same coverage: 89.7\%, with average width 196 against 255 for ACI (difference \(-59\), interval \(-67\) to \(-52\), 23\% narrower), more than Conformal PID with the scorecaster (213). On gasoline and Seattle phase ACI is within the noise of ACI (coverage 86.5\% and 90.0\%), and the plain recency control is indistinguishable from ACI on all four splits, with coverage within 1.4 points.

The German gain matches a visible pattern. The residual scale of that series depends strongly on the day of the week: in the calibration block the mean absolute residual is 109 on Mondays and 21 on Saturdays, a ratio of 5.1 (file \path{results/german_weekday_scale.csv} in the release), and phase weights pick this up (Figure~\ref{fig:weekday}). The rolling-origin interval score on that series is 5\% lower than ACI's (9.21 against 9.72). The result comes from one series and was found after the fact.

\begin{figure}[t]
\centering
\includegraphics[width=\linewidth]{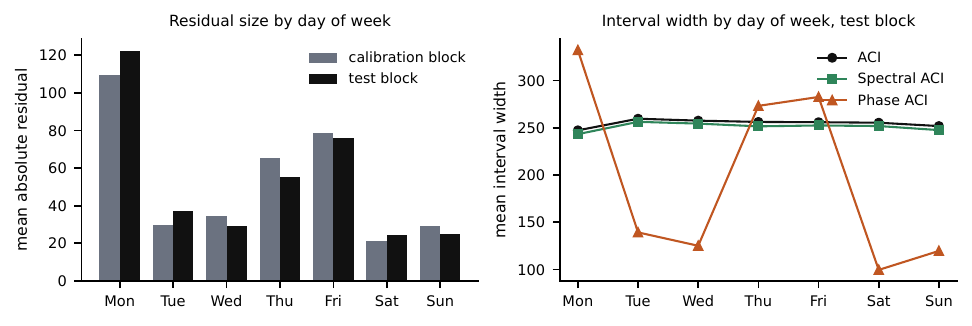}
\caption{German electricity, single split. Left: mean absolute residual of the forecaster by day of week in the calibration and test blocks. Right: mean interval width by day of week in the test block for projected ACI, spectral ACI and phase ACI. Phase ACI widens intervals on the days with large residuals and narrows them on the others, at the same overall coverage (89.7\%).}
\label{fig:weekday}
\end{figure}

\subsection{Simulations}

In the simulations the spectral weights carry information that recency does not (Figure~\ref{fig:sim}; supplement Table~S8). In Simulation~1 the recency control lowers the regime deviation by 0.0010, against 0.0065 for spectral ACI. Phase ACI with period 24 does nothing, because the regime blocks are whole multiples of 24 steps and the phase within 24 steps says nothing about the regime. With period 480, the true recurrence period, it reaches 0.0120, 0.0106 and 0.0122 in Simulations 1, 1b and 3, against 0.0276, 0.0175 and 0.0162 for spectral ACI and 0.0093, 0.0095 and 0.0093 for the oracle. When the recurrence period is known, calendar phase is a better similarity than the spectral features, because it identifies the regime exactly except at block edges and needs no estimation. Under gradual drift (Simulation~2) no weight does much. The rank correction changes little here, because the pools hold 960 points.

\begin{figure}[t]
\centering
\includegraphics[width=\linewidth]{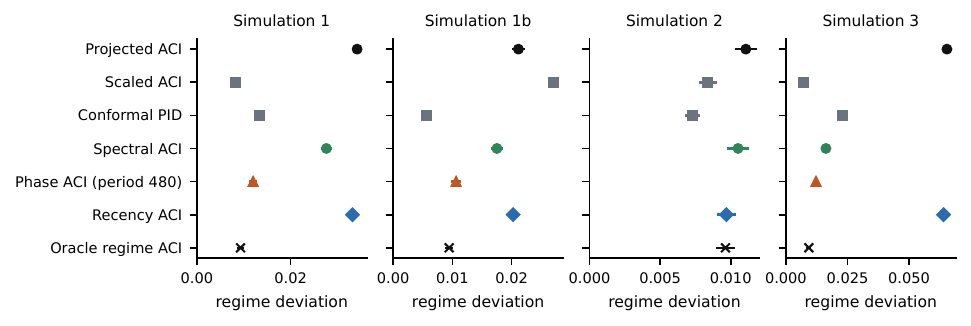}
\caption{Mean absolute deviation of regime coverage from 0.90 in the four simulations (lower is better). Points are means over 250 repetitions and lines are 95\% normal intervals, which are narrower than the markers in most panels. Phase ACI with period 480 and the oracle use the recurrence period or the true regime, and are not available on real data. Phase ACI with period 480 was not run in Simulation 2, which has no recurrence period. Supplement Table~S8 gives the values.}
\label{fig:sim}
\end{figure}

\subsection{What this says}

The spectral idea has content when regimes differ in spectral shape, as in the simulations, and none beyond recency on the nine real series. Season matching helps where the seasonal effect on the residual scale is strong and its period is known, as for the day of the week in German electricity. It is not a general improvement over ACI, and it needs the period, which the spectral features were meant to avoid.

Calendar phase has one theoretical advantage: it satisfies the exogeneity in Assumption~\ref{ass:working}(a). Theorem~\ref{thm:fixed} and Corollary~\ref{cor:barber} would apply to it only if the bandwidth were fixed in advance and, for Theorem~\ref{thm:fixed}, the residuals were also conditionally independent, which real residuals are not. Our bandwidths are chosen on calibration residuals, so no guarantee covers the experiments.

\section{Practical guidance}

The results suggest a short order of work for an analyst with a seasonal series.

\begin{enumerate}
\item Run adaptive baselines first. Projected ACI, scaled ACI and Conformal PID control repaired the failures of split conformal prediction in all our single splits, and the last two were close to the target in the rolling-origin study.
\item Measure conditional calibration by season, regime or phase bin, and not only overall coverage. In Simulation~1 and on the utilities series, overall coverage hid large differences.
\item Compare with a scaled baseline before crediting any similarity-based weights. In our data a plain error-scale tracker did what spectral weights were meant to do, except in Simulation~1b.
\item If spectral weights are used, remember that the features see spectral shape and not season, so they suit regimes that differ in period or in the balance of frequencies. Check the effective sample size at the test points, not only in calibration, and report how often the bandwidth had to be raised. Use a finite-sample rank in the weighted quantile when the pool is small. If season matching is the aim, use calendar phase with the known period, and always compare with a pure recency kernel, which in our data did as well as the spectral weights.
\item Report the length of every test block. An ACI-type long-run identity at \(\gamma=0.005\) says nothing about a block of 100 points.
\item State which choices were made after results were seen.
\end{enumerate}

\section{Limitations}
\label{sec:limits}

\emph{Theory.} The features used in the experiments are computed from past responses, so the fixed-time bound is not a coverage guarantee for the method as run (Remark~\ref{rem:excluded}). Simulation~3 is the only experiment where its assumptions hold, and we designed it. The bandwidth in the main experiments is chosen on calibration residuals, which also breaks the assumptions. Corollary~\ref{cor:barber} needs fixed weights and a fixed level, so it does not cover the adaptive method.

\emph{Evidence for the method.} Spectral weighting improved conditional calibration over equal weights only in simulations with recurring regimes that we designed (Simulations~1, 1b and~3). In each of them a baseline without spectral information did better on that measure, at a cost in width: scaled ACI in Simulations~1 and~3, and Conformal PID in Simulations~1 and~1b. On nine public series we found no benefit. We cannot exclude benefits that our forecasters, features and nine series are too weak to show.

\emph{Features.} The normalized periodogram of a trailing window does not record seasonal phase, and on the real series the weights concentrate on time-adjacent points (Section~\ref{sec:lag}). Calendar phase (Section~\ref{sec:phase}) helps on one series, German electricity, and not on average. That result was found after the fact on one series, and phase weights need a known period, which the spectral features were meant to avoid.

\emph{Quantile convention.} The main spectral results use a weighted quantile whose rank differs from that of projected ACI by about one order statistic, and a check with the corrected rank changes the coverage comparison with ACI (Section~\ref{sec:rolling}). We ran the corrected variants in every experiment, but kept the plain rule as the primary method because it was defined before the results were known.

\emph{Forecasters.} All forecasters are simple linear regressions. With a stronger forecaster the residuals may keep less structure, and the room for any calibration method to help shrinks.

\emph{Implementations.} The AgACI, Conformal PID and scaled ACI results come from our own code. The PID tracker and integrator match the authors' public code on synthetic sequences, but our warm start, scorecaster alignment and constants differ from it, and we did not tune the constants. We did not run EnbPI or SPCI.

\emph{Reporting.} The conditional-calibration measures are post hoc. The gate, Simulation~3 and Simulation~1b were added after earlier results were seen. We report many comparisons without correction. The block and series bootstrap intervals are approximate, and with nine series the series bootstrap is rough.

\emph{Short blocks.} Three of nine series have test blocks of 10 to 31 points. On the six longer series the direction of the results is the same, but those are only 24 origins, and the sign-flip test over six series cannot confirm the interval-score difference between spectral variants and ACI.

\emph{Overlap.} The method shares its core with a companion manuscript (Section~\ref{sec:companion}).

\section{Reproducibility}

All simulations use fixed seeds, and the series are public. The release contains the code, the data, the saved result files, the tests and the pinned package versions. Running \texttt{python run\_all.py} checks the tests and the five original methods. Running \texttt{python run\_all.py -{}-revision} reruns the two original simulations and the calibration effective sample size check, and then rebuilds every other result (several hours for the rolling-origin study and the simulations). The results tables and most of the supplement tables are generated from files in \path{results/} by \path{src/make_revision_tables.py}, and the figures by \path{src/make_figures.py}. The effective sample size table (Table~\ref{tab:ess-diagnostics}), the companion comparison (Table~\ref{tab:companion}), the Conformal PID check table and the calibration effective sample size table in the supplement are typed from saved files or from the companion manuscript. The Conformal PID code can be checked against the public implementation with \path{src/check_pid_against_reference.py}, which needs the authors' file \path{core/methods.py}, downloaded separately. Data sources are listed in \path{DATA_SOURCES.md}.

\section{Discussion}

Spectral similarity is a reasonable idea for choosing which past errors to trust, and it has a clean two-term account: mismatch falls and sampling noise rises as the bandwidth shrinks. That account is correct under exogenous features and independent residuals, and Simulation~3 shows the bound holding with room to spare. It does not carry over to features built from the same series, and it does not show that the idea helps.

The evidence for benefit is narrow. In simulations with recurring regimes, spectral ACI lowered the mean regime deviation by about a fifth relative to ACI when the features came from the response, and by three quarters when they were exogenous. A method that knows the true regime did better still. On real series, feedback methods, and above all a scale tracker and Conformal PID, did what the spectral weights were meant to do, with fewer assumptions. The features cannot match seasons, because the modulus of the Fourier transform ignores phase, and on the real series the weights act like a short recency kernel. A pure recency kernel did at least as well as the spectral weights there, so what looked like spectral similarity was mostly recency and the finite-sample quantile rule. Calendar phase with a known period beat the spectral features in the simulations and gave much narrower intervals on German electricity, which says that the useful structure in these data is a known recurrence, not an estimated spectrum.

A practical reading is that spectral weights are worth trying only after those baselines, and only where regimes differ in spectral shape and the effective sample size stays large. The feature choice, the kernel and the way the weights meet the feedback update were fixed in advance and not tuned here. A learned similarity, or spectral information fed to a score forecaster as in Conformal PID, could behave differently. We have not tested either.

\section*{Author Information}

Jeffery Opoku, University of Texas Rio Grande Valley, School of Mathematical and Statistical Sciences, Edinburg, Texas. Email: \texttt{jeffery.opoku01@utrgv.edu}.

David Banahene, Florida International University, Robert Stempel College of Public Health and Social Work, Department of Biostatistics, Miami, Florida. Email: \texttt{dbana006@fiu.edu}.

\section*{Declarations}

\textbf{Conflict of interest.} The authors declare no competing interests.

\textbf{Funding.} The authors received no specific funding for this work.

\textbf{Ethics approval and consent to participate.} Not applicable. The study uses public time series and simulated data and involves no human participants or animals.

\textbf{Data and code availability.} All nine empirical series are public. Sources are FRED, vega-datasets, Open Power System Data, and the public example datasets in \texttt{statsmodels}. The release includes source details, scripts, tests, pinned package versions, saved result files and the code that builds every table. The release is available at \url{https://github.com/jopoku16/spectral-similarity-conformal} (release v2.0.0).

\textbf{Declaration of generative AI and AI-assisted technologies.} During the preparation of this work, the authors used AI tools for brainstorming ideas and debugging code. The authors reviewed all text, code, analyses, and conclusions, and take full responsibility for the content of the manuscript.

\renewcommand{\bibfont}{\footnotesize}\setlength{\bibsep}{0pt}
\bibliographystyle{plainnat}
\bibliography{references}

\begin{thebibliography}{26}
\providecommand{\natexlab}[1]{#1}
\providecommand{\url}[1]{\texttt{#1}}
\expandafter\ifx\csname urlstyle\endcsname\relax
  \providecommand{\doi}[1]{doi: #1}\else
  \providecommand{\doi}{doi: \begingroup \urlstyle{rm}\Url}\fi

\bibitem[Angelopoulos and Bates(2023)]{angelopoulos2023gentle}
Anastasios~N. Angelopoulos and Stephen Bates.
\newblock A gentle introduction to conformal prediction and distribution-free
  uncertainty quantification.
\newblock \emph{Foundations and Trends in Machine Learning}, 16\penalty0
  (4):\penalty0 494--591, 2023.

\bibitem[Angelopoulos et~al.(2023)Angelopoulos, Cand{\`e}s, and
  Tibshirani]{angelopoulos2023pid}
Anastasios~N. Angelopoulos, Emmanuel~J. Cand{\`e}s, and Ryan~J. Tibshirani.
\newblock Conformal {PID} control for time series prediction.
\newblock In \emph{Advances in Neural Information Processing Systems},
  volume~36, 2023.

\bibitem[Barber et~al.(2021)Barber, Candes, Ramdas, and
  Tibshirani]{barber2021predictive}
Rina~Foygel Barber, Emmanuel~J. Candes, Aaditya Ramdas, and Ryan~J. Tibshirani.
\newblock Predictive inference with the jackknife+.
\newblock \emph{The Annals of Statistics}, 49\penalty0 (1):\penalty0 486--507,
  2021.

\bibitem[Barber et~al.(2023)Barber, Candes, Ramdas, and
  Tibshirani]{barber2023conformal}
Rina~Foygel Barber, Emmanuel~J. Candes, Aaditya Ramdas, and Ryan~J. Tibshirani.
\newblock Conformal prediction beyond exchangeability.
\newblock \emph{The Annals of Statistics}, 51\penalty0 (2):\penalty0 816--845,
  2023.

\bibitem[Bhatnagar et~al.(2023)Bhatnagar, Wang, Xiong, and
  Bai]{bhatnagar2023improved}
Aadyot Bhatnagar, Huan Wang, Caiming Xiong, and Yu~Bai.
\newblock Improved online conformal prediction via strongly adaptive online
  learning, 2023.

\bibitem[{Board of Governors of the Federal Reserve System
  (US)}(2026)]{fredIPG2211A2N}
{Board of Governors of the Federal Reserve System (US)}.
\newblock {Industrial Production: Utilities: Electric and Gas Utilities (NAICS
  = 2211,2) [IPG2211A2N]}.
\newblock Retrieved from FRED, Federal Reserve Bank of St. Louis, 2026.
\newblock URL \url{https://fred.stlouisfed.org/series/IPG2211A2N}.
\newblock Monthly, not seasonally adjusted, index 2017=100.

\bibitem[Dahlhaus(1997)]{dahlhaus1997fitting}
Rainer Dahlhaus.
\newblock Fitting time series models to nonstationary processes.
\newblock \emph{The Annals of Statistics}, 25\penalty0 (1):\penalty0 1--37,
  1997.

\bibitem[Dahlhaus and Polonik(2009)]{dahlhaus2009empirical}
Rainer Dahlhaus and Wolfgang Polonik.
\newblock Empirical spectral processes for locally stationary time series.
\newblock \emph{Bernoulli}, 15\penalty0 (1):\penalty0 1--39, 2009.

\bibitem[Gibbs and Candes(2021)]{gibbs2021adaptive}
Isaac Gibbs and Emmanuel Candes.
\newblock Adaptive conformal inference under distribution shift.
\newblock \emph{Advances in Neural Information Processing Systems},
  34:\penalty0 1660--1672, 2021.

\bibitem[Guan(2023)]{guan2023localized}
Leying Guan.
\newblock Localized conformal prediction: A generalized inference framework for
  conformal prediction.
\newblock \emph{Biometrika}, 110\penalty0 (1):\penalty0 33--50, 2023.

\bibitem[Hore and Barber(2023)]{hore2023conformal}
Ross Hore and Rina~Foygel Barber.
\newblock Conformal prediction with local weights and randomization, 2023.

\bibitem[Lei et~al.(2018)Lei, G'Sell, Rinaldo, Tibshirani, and
  Wasserman]{lei2018distribution}
Jing Lei, Max G'Sell, Alessandro Rinaldo, Ryan~J. Tibshirani, and Larry
  Wasserman.
\newblock Distribution-free predictive inference for regression.
\newblock \emph{Journal of the American Statistical Association}, 113\penalty0
  (523):\penalty0 1094--1111, 2018.

\bibitem[Opoku and Banahene(2026)]{opoku2026dasc}
Jeffery Opoku and David Banahene.
\newblock Drift-aware spectral conformal prediction for non-exchangeable
  streaming data.
\newblock \emph{arXiv preprint arXiv:2606.15953}, 2026.
\newblock URL \url{https://arxiv.org/abs/2606.15953}.

\bibitem[Romano et~al.(2019)Romano, Patterson, and
  Candes]{romano2019conformalized}
Yaniv Romano, Evan Patterson, and Emmanuel Candes.
\newblock Conformalized quantile regression.
\newblock \emph{Advances in Neural Information Processing Systems}, 32, 2019.

\bibitem[Seabold and Perktold(2010)]{seabold2010statsmodels}
Skipper Seabold and Josef Perktold.
\newblock Statsmodels: Econometric and statistical modeling with python.
\newblock In \emph{Proceedings of the 9th Python in Science Conference}, pages
  92--96, 2010.

\bibitem[Stocker et~al.(2025)Stocker, Ma{\l}gorzewicz, Fontana, and
  Ben~Taieb]{stocker2025gentle}
Thomas Stocker, Micha{\l} Ma{\l}gorzewicz, Matteo Fontana, and Souhaib
  Ben~Taieb.
\newblock A gentle introduction to conformal time series forecasting, 2025.

\bibitem[Sun and Yu(2022)]{sun2022copula}
Sophia Sun and Rose Yu.
\newblock Copula conformal prediction for multi-step time series forecasting,
  2022.

\bibitem[Tibshirani et~al.(2019)Tibshirani, Barber, Candes, and
  Ramdas]{tibshirani2019conformal}
Ryan~J. Tibshirani, Rina~Foygel Barber, Emmanuel~J. Candes, and Aaditya Ramdas.
\newblock Conformal prediction under covariate shift.
\newblock \emph{Advances in Neural Information Processing Systems}, 32, 2019.

\bibitem[{U.S. Energy Information Administration}(2026)]{fredGASREGW}
{U.S. Energy Information Administration}.
\newblock {US Regular All Formulations Gas Price [GASREGW]}.
\newblock Retrieved from FRED, Federal Reserve Bank of St. Louis, 2026.
\newblock URL \url{https://fred.stlouisfed.org/series/GASREGW}.
\newblock Weekly, not seasonally adjusted, dollars per gallon.

\bibitem[{Vega Datasets}(2026)]{vegaSeattleWeather}
{Vega Datasets}.
\newblock {Seattle Weather}.
\newblock Retrieved from the vega-datasets repository, 2026.
\newblock URL
  \url{https://github.com/vega/vega-datasets/blob/main/data/seattle-weather.csv}.
\newblock Daily Seattle weather observations.

\bibitem[Wiese et~al.(2019)Wiese, Schlecht, Bunke, Gerbaulet, Hirth, Jahn,
  Kunz, Lorenz, M{"u}hlenpfordt, Reimann, and Schill]{wiese2019open}
Frauke Wiese, Ingmar Schlecht, Wolf-Dieter Bunke, Clemens Gerbaulet, Lion
  Hirth, Martin Jahn, Friedrich Kunz, Casimir Lorenz, Jonathan M{"u}hlenpfordt,
  Juliane Reimann, and Wolf-Peter Schill.
\newblock Open power system data -- frictionless data for electricity system
  modelling.
\newblock \emph{Applied Energy}, 236:\penalty0 401--409, 2019.
\newblock \doi{10.1016/j.apenergy.2018.11.097}.

\bibitem[Xu and Xie(2021)]{xu2021enbpi}
Chen Xu and Yao Xie.
\newblock Conformal prediction interval for dynamic time-series.
\newblock \emph{Proceedings of the 38th International Conference on Machine
  Learning}, pages 11559--11569, 2021.

\bibitem[Xu and Xie(2022)]{xu2022sequential}
Chen Xu and Yao Xie.
\newblock Sequential predictive conformal inference for time series, 2022.

\bibitem[Yang et~al.(2024)Yang, Cand{\`e}s, and Lei]{yang2024bellman}
Zitong Yang, Emmanuel Cand{\`e}s, and Lihua Lei.
\newblock Bellman conformal inference: Calibrating prediction intervals for
  time series, 2024.

\bibitem[Yu et~al.(2025)Yu, Cao, Wang, Yang, Deng, Hu, Luo, and
  Zhou]{yu2025dual}
Qingdi Yu, Zhiwei Cao, Ruihang Wang, Zhen Yang, Lijun Deng, Min Hu, Yong Luo,
  and Xin Zhou.
\newblock Dual-splitting conformal prediction for multi-step time series
  forecasting, 2025.

\bibitem[Zaffran et~al.(2022)Zaffran, Dieuleveut, F{\'e}ron, Goude, and
  Josse]{zaffran2022adaptive}
Margaux Zaffran, Aymeric Dieuleveut, Olivier F{\'e}ron, Yannig Goude, and Julie
  Josse.
\newblock Adaptive conformal predictions for time series.
\newblock \emph{Proceedings of the 39th International Conference on Machine
  Learning}, pages 25834--25866, 2022.

\end{thebibliography}

\end{document}